\documentclass[sigconf, balance=false]{acmart}
\usepackage{popets}

\setcopyright{popets}
\copyrightyear{2023}
\acmYear{2023}
\acmVolume{2023}
\acmNumber{2}
\acmDOI{10.56553/popets-2023-0041}
\acmISBN{}
\acmConference{Proceedings on Privacy Enhancing Technologies}
\usepackage[utf8]{inputenc}
\usepackage{multirow}
\usepackage{dsfont}
\usepackage{makecell}
\usepackage{booktabs}
\usepackage{algorithm}
\usepackage{algorithmic}
\usepackage{bm}
\usepackage{xcolor, colortbl}
\usepackage{enumitem}
\usepackage{microtype}
\usepackage{graphicx}
\usepackage{subfigure}
\usepackage{xspace}
\usepackage{enumitem}
\usepackage{adjustbox}
\newcommand{\mpara}[1]{\medskip\noindent{\bf #1}}

\usepackage{pifont}
\newcommand{\cmark}{\ding{51}}
\newcommand{\xmark}{\ding{55}}

\newcommand{\loss}[1]{\ensuremath{\mathcal{#1}}}
\newcommand{\func}[1]{\ensuremath{\mathsf{#1}}}
\newcommand{\mA}{\ensuremath{\bm{A}}}
\newcommand{\mD}{\ensuremath{\bm{D}}}
\newcommand{\mX}{\ensuremath{\bm{X}}}
\newcommand{\mY}{\ensuremath{\bm{Y}}}

\newtheorem{lemma}{Lemma}
\newtheorem{hypothesis}{Hypothesis}

\newtheorem{RQ}{RQ}
\counterwithin*{RQ}{section} 

\newcommand{\faac}{\ensuremath{\textsc{GSEF-concat}}\xspace}
\newcommand{\faaw}{\ensuremath{\textsc{GSEF-mult}}\xspace}
\newcommand{\exhoney}{\ensuremath{\textsc{GSEF}}\xspace}
\newcommand{\exohoneyex}{\ensuremath{\textsc{GSE}}\xspace}
\newcommand{\exoattrsim}{\ensuremath{\textsc{ExplainSim}}\xspace}

\newcommand{\attrsim}{\textsc{FeatureSim}\xspace}
\newcommand{\lsa}{\textsc{Lsa}\xspace}
\newcommand{\graphmi}{\textsc{GraphMI}\xspace}
\newcommand{\gsl}{\textsc{Slaps}\xspace}

\newcommand{\cora}{\textsc{Cora}\xspace}
\newcommand{\coraml}{\textsc{CoraML}\xspace}
\newcommand{\bitcoin}{\textsc{Bitcoin}\xspace}
\newcommand{\cseer}{\textsc{CiteSeer}\xspace}
\newcommand{\pubmed}{\textsc{PubMed}\xspace}
\newcommand{\credit}{\textsc{Credit}\xspace}

\newcommand{\grad}{\textsc{Grad}\xspace}
\newcommand{\gradinput}{\textsc{Grad-I}\xspace}
\newcommand{\zorro}{\textsc{Zorro}\xspace}
\newcommand{\szorro}{\textsc{Zorro-S}\xspace}
\newcommand{\glime}{\textsc{GLime}\xspace}
\newcommand{\gnnexp}{\textsc{GNNExp}\xspace}
\usepackage{tcolorbox}

\begin{document}

\title{Private Graph Extraction via Feature Explanations}

\author{Iyiola E. Olatunji}
\affiliation{%
  \institution{L3S Research Center}
  \city{Hannover}
  \country{Germany}}
\email{iyiola@l3s.de}

\author{Mandeep Rathee}
\affiliation{%
  \institution{L3S Research Center}
  \city{Hannover}
  \country{Germany}}
\email{rathee@l3s.de}   

\author{Thorben Funke}
\affiliation{%
  \institution{L3S Research Center}
  \city{Hannover}
  \country{Germany}}
 \email{tfunke@l3s.de}

\author{Megha Khosla}
\affiliation{%
 \institution{TU Delft}
 \city{Delft}
 \country{Netherlands}}
 \email{m.khosla@tudelft.nl}

\renewcommand{\shortauthors}{Olatunji et al.}

\begin{abstract}
Privacy and interpretability are two important ingredients for achieving trustworthy machine learning. We study the interplay of these two aspects in graph machine learning through graph reconstruction attacks. The goal of the adversary here is to reconstruct the graph structure of the training data given access to model explanations.
Based on the different kinds of auxiliary information available to the adversary, we propose several graph reconstruction attacks. 
We show that additional knowledge of post-hoc feature explanations substantially increases the success rate of these attacks. Further, we investigate in detail the differences between attack performance with respect to three different classes of explanation methods for graph neural networks: gradient-based, perturbation-based, and surrogate model-based methods. While gradient-based explanations reveal the most in terms of the graph structure, we find that these explanations do not always score high in utility. For the other two classes of explanations, privacy leakage increases with an increase in explanation utility.  
Finally, we propose a defense based on a randomized response mechanism for releasing the explanations, which substantially reduces the attack success rate. Our code is available at \url{https://github.com/iyempissy/graph-stealing-attacks-with-explanation}.

\end{abstract}


\maketitle

\section{Introduction}
Graphs are highly informative, flexible, and natural ways of representing data in various real-world domains. 
Graph neural networks (GNNs) \cite{kipf2017semi,hamilton2017inductive,velickovic2018graph} have emerged as the standard tool to analyze graph data that is non-euclidean and irregular in nature. GNNs have gained state-of-the-art results in various graph analytical tasks ranging from applications in biology, and healthcare \cite{dong2022mucomid, Ahmedt2021graphmedical} to recommending friends in a social network \cite{fan2019graph}. GNNs' success can be attributed to their ability to extract powerful latent features via complex aggregation of neighborhood aggregations \cite{funke2021zorro, ying2019:gnnexplainer}.

However, these models are inherently black-box and complex, making it extremely difficult to understand the underlying reasoning behind their predictions.
With the growing adoption of these models in various sensitive domains, efforts have been made to explain their decisions in terms of feature as well as neighborhood attributions. Model explanations can offer insights into the internal decision-making process of the model, which builds the trust of the users. Moreover, owing to the current regulations \cite{officialgrpr2016} and guidelines for designing trustworthy AI systems, several proposals advocate for deploying (automated) model explanations \cite{goodman2017european, selbst2018meaningful}.

Nevertheless, releasing additional information, such as explanations, can have adverse effects on the privacy of the training data. While the risk to privacy due to model explanations exists for machine learning models in general \cite{shokri2021privacy}, it can have more severe implications for graph neural networks. For instance, several works \cite{olatunji2021membership,du2018towards} have established the increased vulnerability of GNNs to privacy attacks due to the additional encoding of graph structure in the model itself. We initiate the \emph{first investigation} of the effect of releasing feature explanations for graph neural networks on \emph{the leakage of private information} in the training data.  

\begin{figure}
    \centering
    \includegraphics[width=0.48\textwidth]{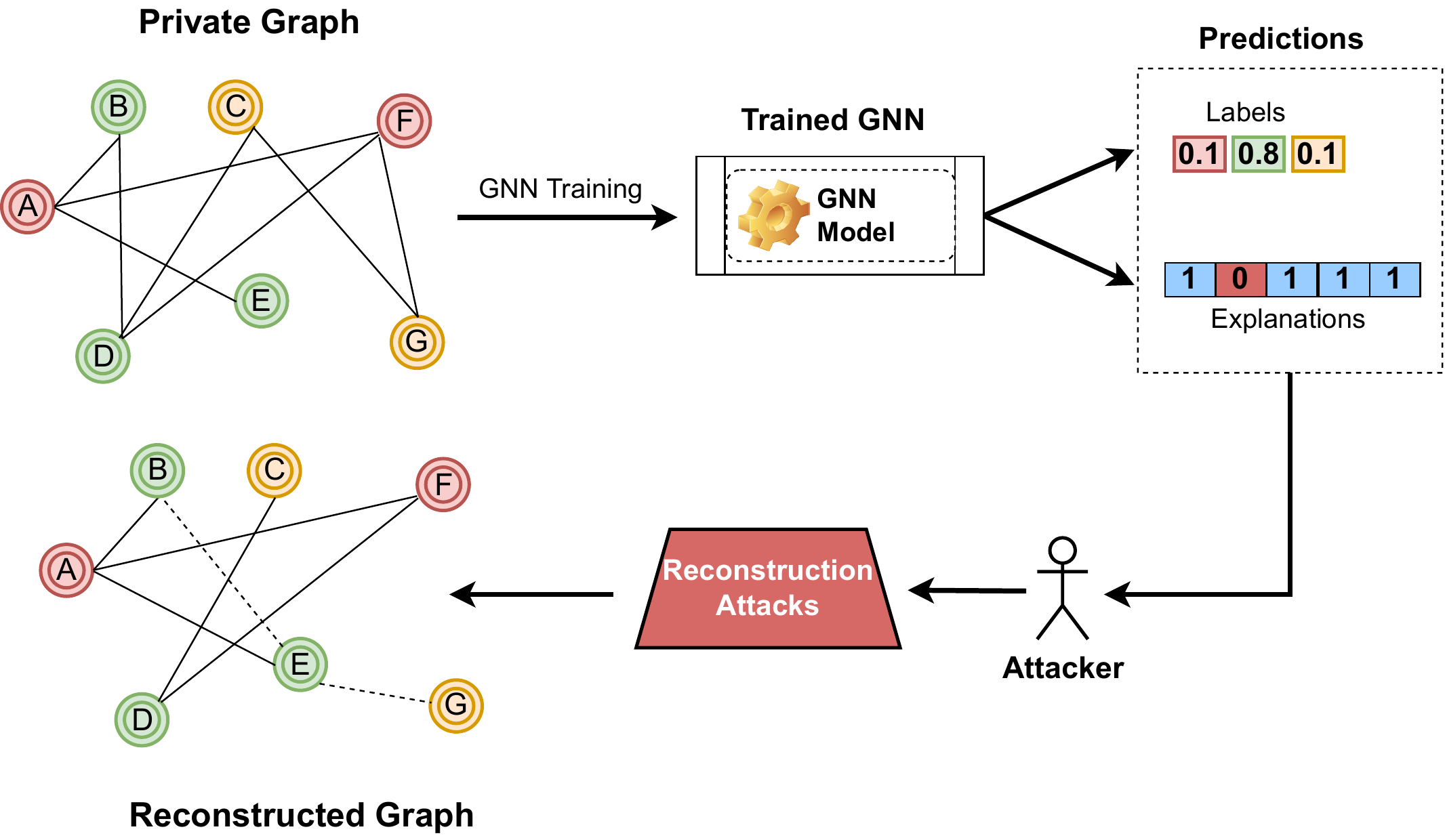}
    \caption{The importance scores for the features, as provided by an explanation, can be exploited to infer the graph structure. Here, we show an example of a binary explanation where a score of $1$ indicates that the corresponding feature is part of the explanation.}
    \label{fig:motivating-example}
\end{figure}

To analyze the information leakage due to explanations, we take the perspective of an adversary whose goal is to infer the hidden connections among the training nodes. 
Consider a setting where the user has access to node features and labels but not the graph structure among the nodes. For example, node features could be part of the public profile of various individuals. The graph structure among the nodes could be some kind of social network which is private. Now, let's say that the user wants to obtain a trained GNN on her data while providing the node features and labels to a central authority that has access to the private graph structure. We ask here the question: \emph{how much do the feature-based explanations leak information about the graph structure used to train the GNN model?} We quantify this information leakage via several graph reconstruction attacks. A visual illustration is shown in Figure \ref{fig:motivating-example}. Specifically, our threat model consists of two main settings: first, where the adversary has access only to feature explanations, and second, where the adversary has additional information on node features/labels. Note that we only focus on feature-based explanations for GNNs in this work. For other explanation types such as node or edge explanations, as the returned nodes and edges belong to the node's neighborhood, it becomes trivial to reconstruct the graph structure.

\subsection{Our Contributions and Findings} Ours is the first work to analyze the risks of releasing feature explanations in GNNs on the privacy of the relationships/connections in the training nodes. 
We quantify the information leakage via explanations by the success rate of several graph reconstruction attacks. Our attacks range from a simple explanation similarity-based attack to more complex attacks exploiting graph structure learning techniques. Besides, we provide a thorough analysis of information leakage via feature-based explanations produced by three classes of GNN post-hoc explanation methods, including \emph{gradient-based}, \emph{perturbation-based}, and \emph{surrogate model-based}. 
 
To analyze the differences in the robustness of the explanation methods to our privacy attacks, we investigate the explanation utility in terms of \emph{faithfulness} and \emph{sparsity}.  We find that the gradient-based methods are the most susceptible to graph reconstruction attacks even though the corresponding explanations are the least faithful to the model. In other words, they have low utility. This is an important finding as the corresponding explanations could release a large amount of private information without offering any utility to the user.
The perturbation-based approach \zorro and its variant show the highest explanation utility as well as a high success rate of the attack, pointing to the expected trade-off between privacy and explanation utility. 

We perform our study over three types of datasets with varying properties. For instance, our first dataset (\cora) has a large number of binary features, but the feature space is very sparse. Our second dataset (\coraml) has fewer but denser features. Our final dataset (\bitcoin) has a very small number of features.  
We find that the information leakage varies with explanation techniques as well as the feature size of the dataset. 
The dataset (\bitcoin) with the smallest feature size (8) is the most difficult to attack. Through experiments on three additional datasets (\cseer, \pubmed, and \credit), we later show that although a small feature size might limit the information leakage, the correlation between node features with node connections plays a significant role.
All baseline attacks which rely on knowledge of only features and labels perform no better than a random guess. In such a case, explanation-based attacks provide an improvement, though not very huge, in inferring private graph structure information. We also perform additional experiments to evaluate the effectiveness of the attacks on three other datasets. Also, we perform a subgraph reconstruction attack when only a subset of the explanations is available to the attacker.

Finally, we develop a perturbation-based defense for releasing feature-based explanations. Our defense employs a randomized response mechanism to perturb the individual explanation bits. We show that our defense reduces the attack to a random guess with a small drop in the explanation utility.  Our code is available at \url{https://github.com/iyempissy/graph-stealing-attacks-with-explanation}.

\section{Preliminaries}
\subsection{Graph Neural Networks}
Graph neural networks (GNNs) \cite{kipf2017semi,hamilton2017inductive,velickovic2018graph} are a special family of deep learning models designed to perform inference on graph-structured data. The variants of GNNs, such as graph convolutional network (GCN), compute the representation of a node by utilizing its feature representation and that of the neighboring nodes. That is, GNNs compute node representations by recursive aggregation and transformation of feature representations of its neighbors.

Let $G = (V, E)$ denote a graph where $V$ is the node-set, and $E$ represents the edges or links among the nodes. Furthermore, let $\boldsymbol{x}_{i}^{(\ell)}$ be the feature representation of node $i$ at layer $\ell$, ${\mathcal{N}}_{i}$ denote the set of its 1-hop neighbors and $\bm{\theta}$ is a learnable weight matrix.
Formally, the $\ell$-th layer of a graph convolutional operation can be described as
\begin{align}
    \boldsymbol{z}_{i}^{(\ell)}=&\operatorname{AGGREGATION}^{(\ell)}\left(\left\{\boldsymbol{x}_{i}^{(\ell-1)},\left\{\boldsymbol{x}_{j}^{(\ell-1)} \mid j \in{\mathcal{N}}_{i}\right\}\right\}\right) \\
    \boldsymbol{x}_{i}^{(\ell)}= &\operatorname{TRANSFORMATION} ^{(\ell)}\left(\boldsymbol{z}_{i}^{(\ell)}\right)
\end{align}
Then a softmax layer is applied to the node representations at the last layer (say $L$) for the final prediction of the node classes $\mathcal{C}$
 \begin{align}\label{eq:pred}
     \boldsymbol{y}\leftarrow \operatorname{argmax}(\operatorname{softmax}(\boldsymbol{z}_{i}^{(L)}\bm{\theta})),
 \end{align}
GNNs have been shown to possess increased vulnerability to privacy attacks due to encoding of additional graph structure in the model \cite{olatunji2021membership,du2018towards}. We further investigate the privacy risks of releasing post-hoc explanations for GNN models.

\subsection{Explaining Graph Neural Networks}
GNNs are deep learning models which are inherently black-box or non-interpretable. This black-box behavior becomes more critical for applying them in sensitive domains like medicine, crime, and finance. Consequently, recent works have proposed post-hoc explainability techniques to explain the decisions of an already-trained model. In this work, we are concerned with the task of node classification, where the model is trained to predict node labels in a graph. For such a task, an instance is a single node. An explanation for a node usually consists of a subset of the most important features as well as a subset of its neighboring nodes/edges responsible for the model's prediction. Depending on the explanation method, the importance is usually quantified either as a continuous score (also referred to as a soft mask) or a binary score (also called a hard mask).
In this work, we consider three popular classes of explanation methods: \emph{gradient-based} \cite{selvaraju2017grad,baldassarre2019explainability,pope2019explainability}, \emph{perturbation-based} \cite{ying2019:gnnexplainer, funke2021zorro}, and \emph{surrogate} \cite{huang2020graphlime} methods. 

\mpara{Gradient-based Methods.} These approaches usually employ the gradients of the target model's prediction with respect to input features as importance scores for the node features. We use two gradient-based methods in our study, namely \textbf{\grad}\cite{imageclasssimonyan2013deep} and Grad-Input (\textbf{\gradinput}) \cite{sundararajan2017:integratedgrad}. For a given graph $G$ and the trained GNN model $f(\boldsymbol{X},G,\theta)$ (where $\theta$ is the set of parameters and $\boldsymbol{X}$ is the features matrix), \grad generates an explanation $\mathcal{E}_X$ by assigning continuous valued importance scores to the features. For node $i$ and $\boldsymbol{x}_{i}$ is its features vector, The score is calculated by $\frac{\partial f}{\partial \boldsymbol{x}_{i}}$. \gradinput transforms \grad explanation by an element-wise multiplication with the input features ($\boldsymbol{x}_{i}\odot\frac{\partial f}{\partial \boldsymbol{x}_{i}}$).

\mpara{Perturbation-based Methods.} 
Perturbation-based methods obtain soft or hard masks over the features/nodes/edges as explanations by monitoring the change in prediction with respect to different input perturbations.
We use two methods from this class: \textbf{\zorro} \cite{funke2021zorro} and \textbf{\gnnexp} \cite{ying2019:gnnexplainer}. \zorro learns discrete masks over input nodes and node features as explanations using a greedy algorithm. It optimizes a fidelity-based objective that measures how the new predictions match the original predictions of the model by fixing the selected nodes/features and replacing the others with random noise values. This returns a hard mask for the explanations. 
\gnnexp learns soft masks over edges and features by minimizing the cross-entropy loss between the predictions of the original graph and the predictions of the newly obtained (masked) graph.
We also utilize the \zorro variant that provides soft explanation masks called \textbf{\szorro}. \szorro relaxes the argmax in \zorro's objective with a softmax, such that the masked are retrievable with standard gradient-based optimization. Together with the regularization terms of \gnnexp, \szorro learns sparse soft masks.

\mpara{Surrogate Methods.}
Surrogate methods fit a simple and interpretable model to a sampled local dataset corresponding to the query node. For example, the sampled dataset can be generated from the neighbors of the given query node. The explanations from the surrogate model are then used to explain the original predictions. We use GraphLime\cite{huang2020graphlime}, which we denote as \textbf{\glime} from this class. As an interpretable model, it uses the global feature selection method HSIC-Lasso~\cite{yamada2014high}. The sampled dataset consists of the node and its neighborhood. The set of the most important features returned by HSIC-Lasso is used as an explanation for the GNN model.

\textit{Remark:} Please note that in this work, we assume that only feature-based explanations are released to the user. In the presence of  node/edge explanations, an adversary can trivially reconstruct large parts of the neighborhood as the returned nodes/edges are part of the node's original neighborhood.

\subsubsection{Measuring Explanation Quality.} 
We measure the quality of the explanation by its ability to approximate the model's behavior which is referred to as \textbf{\textit{faithfulness}}. As the groundtruth for explanations is not available, we use the \textit{RDT-Fidelity} proposed by \cite{funke2021zorro} to measure faithfulness. The corresponding fidelity score measure how the original and new predictions match by fixing the selected nodes/features and replacing the others with random noise values.
Formally, the \textit{RDT-Fidelity} of explanation $\mathcal{E}_X$ corresponding to explanation mask $M(\mathcal{E}_X)$ with respect to the GNN $f$ and the noise distribution $\mathcal{N}$ is given by
\begin{equation}
\label{eq:fidelity}
\mathcal{F}(\mathcal{E}_X) = \mathbb{E}_{Y_{\mathcal{E}_X}|Z\sim \mathcal{N}} \left[\mathds{1}_{f\left(X\right)=f(Y_{\mathcal{E}_X})}\right],
\end{equation}
where the perturbed input is given by
\begin{equation}
 \tilde{I}_{\mathcal{E}_X} = X\odot M(\mathcal{E}_X) + Z\odot(\mathds{1} - M(\mathcal{E}_X)), Z\sim \mathcal{N},
    \label{eq:noisy_features_matrix}
\end{equation}
where $\odot$ denotes an element-wise multiplication, and $\mathds{1}$ is a matrix of ones with the corresponding size. 

\mpara{Sparsity.} We further note that by definition, the complete input is faithful to the model. Therefore, we further measure the sparsity of the explanation.  A meaningful explanation should be sparse and only contain a small subset of features most predictive of the model decision. We use the entropy-based sparsity definition from \cite{funke2021zorro} as it is applicable for both soft and hard explanation masks. Let $p$ be the normalized distribution of explanation (feature) masks. Then the sparsity of an explanation is given by the entropy $H(p)$ over the mask distribution, 
$$H(p)= -\sum_{f\in M} p(f) \log p(f).$$
Note that the entropy here is bounded by $\log(|M|)$, where $M$ corresponds to the size of the feature set. The lower the entropy, the sparser the explanation. We will utilize these two metrics used in measuring explanation quality to argue about the differences in the attack performance.

\section{Threat Model and Attack Methodology}
\subsection{Motivation}
We consider the setting in which different data holders hold the features and graph (adjacency matrix) in practice. Specifically, the central trusted server has access to the graph structure and trains a GNN model using the node features and labels provided by the user.
The user can further query the trained GNN model by providing node features. As the features are already known to the user, revealing \textit{feature explanations} for the prediction might be considered a safe way to increase the user's trust in the model. We investigate such a scenario and uncover the increased privacy risks of releasing \emph{feature explanations even if the original node features/labels are already known to the adversary}.

Our setting is inspired by previous works and practical scenarios \cite{zhanggraphmi, fabijanlacksharing, barua2007enabling, yang2011information, greenstein_2020}. 
We elaborate on one such scenario in the following. Consider a multinational company with several subdivisions.
For simplicity, we consider two subdivisions (the marketing and product departments).
The marketing department collects the interaction between users (edge information), while the product department collects data on user behavior (features).
Under privacy law(e.g., GDPR), such user interactions cannot be shared because they are collected for different purposes.
However, they aim to train a GNN model jointly since it would benefit both teams.
The product department then sends the node features to the marketing department, which in addition, uses their edge information to train a predictive model and releases an API (similar to MLaaS).
The API returns both the prediction for the features and the corresponding feature explanations for such predictions.
The product department, having observed how useful the exploratory analysis is, aims to recover the private edges using the information returned by the API (explanations) and their node features.

\subsection{Threat Model}
We consider two scenarios: \textbf{first}, in which the adversary has access to node features and/or labels and obtains additional access to feature explanations, and \textbf{second}, in which the adversary has access only to feature explanations. The goal of the adversary is to infer the private connections of the graph used to train the GNN model. 

Corresponding to the above settings, we categorize our attacks as \emph{explanation augmentation} (when explanations are augmented with other information to launch the attack) and \emph{explanation-only} attacks. Through our five proposed attacks, we investigate the privacy leakage of the explanations generated by six different explanation methods. 

Our simple similarity-based explanation-only attack already allows us to quantify the additional information that the feature-based explanation encodes about the graph structure. Our explanation augmentation attacks are based on the graph structure learning paradigm, which allows us to effectively integrate additional known information in learning the private graph structure. Besides, our explanation augmentation attacks also result in a successfully trained GNN model without the knowledge of the true graph structure, offering an additional advantage to the adversary.

\subsection{Attack Methodologies}
\label{sec:attackmeth}
Here, we provide a detailed description of our two types of attacks. We commence with the \textit{explanation-only attack } in which we utilize only the provided explanation to launch the attack, followed by the \textit{explanation augmentation} attacks in which more information such as node labels or/and features are exploited in addition to the explanation. The taxonomy of our attacks based on the attacker's knowledge is presented in Table \ref{tab:attack-settings}.

\mpara{Explanation-only Attack.}
\begin{figure}[h!]
    \centering
    \includegraphics[width=0.3\textwidth]{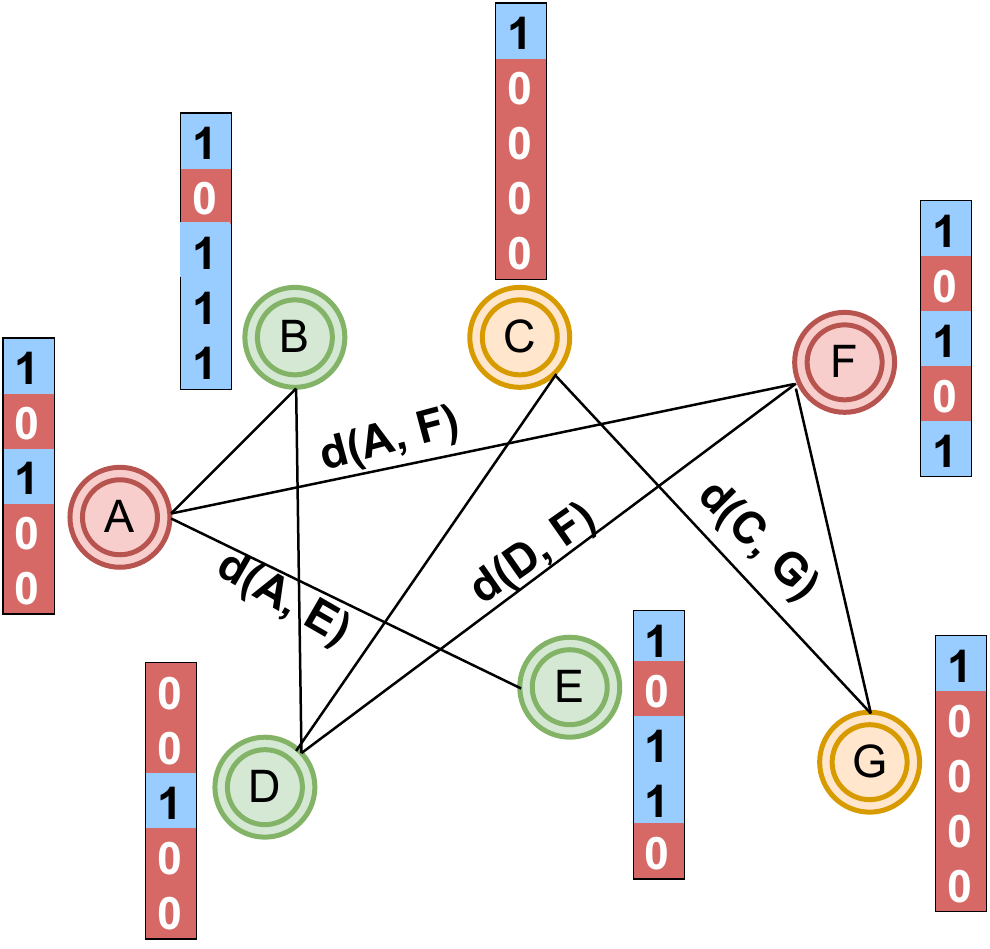}
    \caption{\exoattrsim attack. Each node is assigned a feature explanation vector where blue (1) and red (0) indicate whether the feature is part of the explanation or not. The attacker then assigns an edge by computing the pairwise similarity represented as $d(node_i, node_j)$ between nodes' explanation vectors. We show the representation of \zorro's explanation for easier visualization.}
    \label{fig:explainsim}
\end{figure}
This is an unsupervised attack in which the attacker only has access to the explanations and does not have access to the features or the labels. The attacker measures the distance between each pair of explanation vectors and assigns an edge between them if their distance is small. The intuition is that the model might assign similar labels to connected nodes, which leads to similar explanations. We experimented with various distance metrics, but cosine similarity performs best across all datasets. We refer to this similarity-based attack as \textbf{\exoattrsim}. This attack is illustrated in Figure \ref{fig:explainsim}. We report the performance using the area under the receiver operating characteristic curve (AUC) and average precision (AP). Rather than choosing one threshold which could be domain-dependent, these metrics allow us to provide a holistic evaluation across the complete range of thresholds.

\mpara{Explanation Augmentation Attacks.}
\begin{figure*}[h!]
   \centering
   \includegraphics[width=0.8\textwidth]{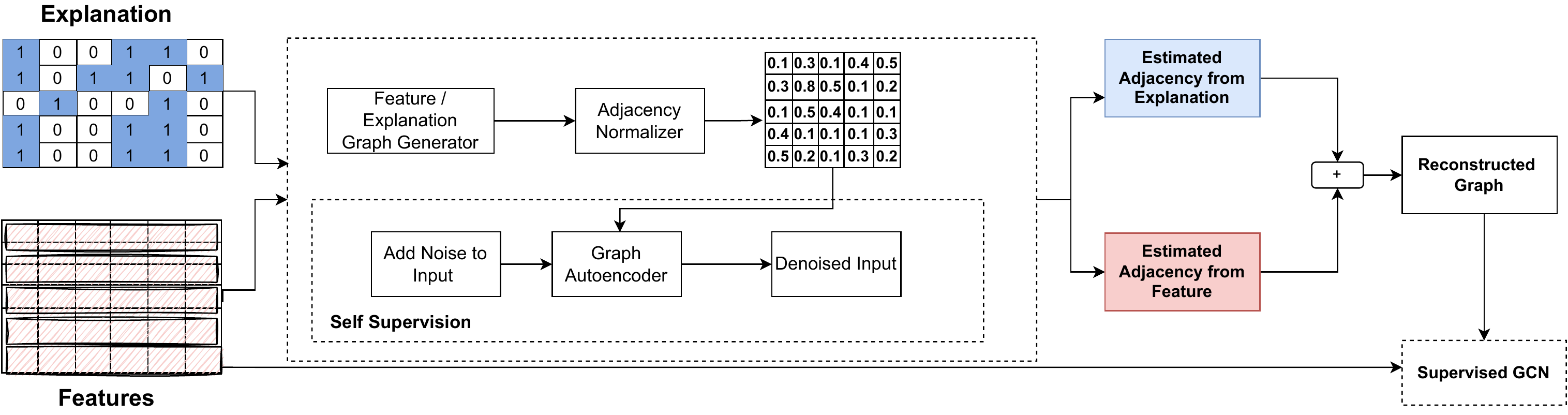}
   \caption{%
   \label{fig:exhoney} %
    Overview of \exhoney. The generator takes node features and explanations as input and outputs an adjacency matrix which may be non-normalized by the adjacency normalizer. The normalized adjacency matrix is used in predicting both the class labels and reconstructing the node features (explanations) by the denoising autoencoders. The final reconstructed adjacency is the one that minimizes the reconstruction error on each of the node features and explanations and the loss of the class label prediction.
    }
\end{figure*}
Towards explanation augmentation attacks, we leverage the graph structure learning paradigm of \cite{fatemi2021slaps}. In particular, we employ two \textit{generator modules} for generating graph edges corresponding to features and explanations, respectively. The generators are trained using feature/explanation reconstruction-based losses and node classification loss. We commence by describing the common architecture of the attack model followed by its concrete usage in the four attack variations in Section \ref{sec:augmentattacks}.

\paragraph{\textbf{Generators}} The two generators take the node features/explana\-tions as input and output two adjacency matrices. 
We employ the full parameterization (FP) approach to model the generators similar to those used in \cite{franceschi2019learning, fatemi2021slaps}.
In other words, each element of the adjacency matrix $\tilde{\mA}$ is treated as a separate learnable parameter, i.e., the adjacency matrix is fully parameterized.
We initialize these learnable parameters by the cosine similarity among the feature/explanation vectors of the corresponding nodes. Let the output adjacency matrix function be given as $\tilde{\mA}=\func{G}_{FP}(\mX; \bm{\theta}_\func{G})=\bm{\theta}_\func{G}$ where $\bm{\theta}_\func{G}\in\mathbb{R}^{n\times n}$ and $\func{G}_{FP}(\cdot;\cdot)$ denotes the generator function.
To obtain a symmetric adjacency matrix with all positive elements, we perform the following transformation
\begin{equation*}
    \mA = \mD^{-\frac{1}{2}}\Big(\frac{\func{P}_{[0,1]}(\tilde{\mA})+\func{P}_{[0,1]}(\tilde{\mA})^T}{2}\Big)\mD^{-\frac{1}{2}}, 
\end{equation*}
where \func{P} is a non-negative function defined by 
\begin{equation*}
\func{P}_{[0,1]}[x]=\left\{
\begin{array}{rcl}
0 & & x < 0,\\
1 & & x > 1,\\
x & & otherwise.
\end{array} \right.
\label{eqn:projection}
\end{equation*}
The final adjacency matrix is computed by adding the matrices corresponding to the two generators. Any element greater than one is trimmed to 1. In other words, a cell in the adjacency matrix follows a Bernoulli distribution.
The final graph is then generated by sampling from the learned Bernoulli distributions. More precisely, we select edge ($i,j$) with probability $p$ where $p\in [0,1]$ is the learnt weight for cell $(i,j)$.

\paragraph{ \textbf{Training With Self-supervision and Node Classification Loss}}
We remark that only some nodes used to train the attack model are labeled. To learn effective representations of both labeled and unlabeled nodes, we employ a self-supervision module in addition to training with the supervised classification loss. In particular, for the self-supervision module, we use \textit{denoising graph autoencoders}, which takes noisy features/explanations and the graph sampled from the generator as input. The goal is to reconstruct the true node features and explanations.

Let $\func{{DAE}}$ and $\func{{DAE_{\mathcal{E}_X}}}$ denote the feature and explanation reconstruction modules respectively. 
They take the noisy node features(explanations) as input and output a denoised version of the features (explanations) with the same dimension. 
To generate the noisy features/explanations, we first select random indices. We then add random noise to feature/explanation values corresponding to the selected indices. Concretely, for binary features/explanations, we randomly flip $r=20\%$ of the index with original values $1$ to $0$. For features and explanations with continuous values, we add independent Gaussian noise to $r=20\%$ percent of feature/explanation elements (corresponding to selected indices). To train our self-supervision module, we employ binary cross entropy loss for binary features/explanations and mean squared error for continuous features/explanations. Let $\loss{L}_{DAE}$ and  $\loss{L}_{DAE_{\mathcal{E}_X}}$ denote the loss functions corresponding to feature and explanation reconstruction, respectively.

For supervised training, we employ a graph convolution network (GCN) which takes as input the node features and the graph sampled from the generator to predict node labels. Note that the GCN used here is not the target model. We use the cross-entropy loss ($CE$) between the predicted labels ($\tilde{Y}$) and the ground-truth labels ($Y$), $\loss{L}_C  = CE(Y, \tilde{Y})$.

The final training loss for private graph extraction is then given by $$\loss{L}= \loss{L}_{DAE} + \loss{L}_{DAE_{\mathcal{E}_X}} + \loss{L}_C.$$

We refer to the above attack framework as \textbf{G}raph \textbf{S}tealing with \textbf{E}xplanations and \textbf{F}eatures (\textbf{\exhoney}), and the schematic diagram is given in Figure \ref{fig:exhoney}. Besides, we have three attack variations that employ a single generator module, as described below.

\subsubsection{Attack Variations}
\label{sec:augmentattacks}
\begin{table}[h!]
\caption{Attack taxonomy based on attacker's knowledge of node features ($\mX$), labels ($\mY$) and feature explanations ($\mathcal{E}_X$).}
\label{tab:attack-settings}
\begin{tabular}{llll}
\toprule
\textsc{Attack} & {$\mX$} & {$\mY$} & $\mathcal{E}_X$ \\ \midrule
\exoattrsim    & \xmark                 & \xmark               & \cmark                                       \\
\exhoney       & \cmark                 & \cmark               & \cmark \\
\faac        & \cmark                 & \cmark               & \cmark                                       \\ 
\faaw        & \cmark                 & \cmark               & \cmark                                  \\

\exohoneyex    & \xmark                 & \cmark               & \cmark                                       \\ 
 \bottomrule
\end{tabular}
\end{table}

Besides \exhoney, we have three attack variations that employ explanations (i) \textbf{{\faac}} in which we concatenate the node features and explanations and feed the concatenated input to a single generator module (ii) \textbf{\faaw} in which we perform element-wise multiplication between the features and the explanations and feed them into a single graph generator module. This is equivalent to assigning importance to the node features, which emphasize the essential characteristics of the nodes. Similar to \faac, we reconstruct the adjacency matrix using one generator of Figure \ref{fig:exhoney} and (iii) \textbf{\exohoneyex}  in which the attacker only has access to the explanations and labels. Here, we also employ only one generator with explanations as input. 

\section{Experiments}
In this section, we present the experimental results to show the effectiveness of explanation-based attacks. Specifically, our experiments are designed to answer the following research questions:

\begin{RQ}
How does the knowledge of feature explanations influence the reconstruction of private graph structure?
\label{rq:attack-performance}
\end{RQ}
\begin{RQ}
\label{rq:exp-vulnerability}
What are the differences between explanation methods with respect to privacy leakage?
\end{RQ}

\begin{RQ}
What is the additional advantage of the adversary (for example, in terms of the utility of the inferred information on a downstream task) on explanation augmentation attacks? 
\label{rq:reconstruct-edges}
\end{RQ}

\begin{RQ}
How much does the lack of knowledge about groundtruth node labels affect attack performance?
\label{rq:targetmodelaccess}
\end{RQ}

\subsection{Experimental Settings}

\subsubsection{Attack Baselines Without Explanations}
\paragraph{\textbf{\attrsim}} In this unsupervised attack, the attacker computes the pairwise similarity between pairs of the actual features to reconstruct the graph. Specifically, an edge exists between two nodes if the distance between their feature representation is low. We use cosine similarity as a measure of similarity because it performs better than other distance metrics.

\paragraph{\textbf{\lsa} \cite{he2021stealing}} 
\lsa (Link stealing attack) is a black-box attack that assumes that the attacker has access to a dataset drawn from a similar distribution to the target data (shadow dataset). Additionally, \lsa knows the architecture of the target model and can train a corresponding shadow model that replicates the behavior of the target model. The goal of the attack is to infer sensitive links between nodes of interest. We compare our results with their proposed attack-2, where an attacker has access to the node features and labels. They trained a separate MLP model (reference model) using the available target attributes and their corresponding labels to obtain posteriors. Then, \lsa computes the pairwise distance between posteriors obtained from the target model and that of the reference model for the nodes of interest. We use cosine similarity as the distance metric.

\paragraph{\textbf{\graphmi} \cite{zhanggraphmi}} \graphmi is a white-box attack in which the attacker has access to the parameters of the target model, node features, all the node labels, and other auxiliary information like edge density. The goal of an attacker is to reconstruct the sensitive links or connections between nodes. 
The attack model uses the cross entropy loss between the true labels and the output posterior distribution of the target model, along with feature smoothness and adjacency sparsity constraints, to train a fully parameterized adjacency matrix. The graph is then reconstructed using the graph autoencoder module in which the encoder is replaced by learned parameters of the target model, and the decoder is a logistic function.

\paragraph{\textbf{\gsl} \cite{fatemi2021slaps}}
Since our attack model is built on top of the graph structure learning framework of SLAPS, we performed an experiment using the vanilla SLAPS. Given node features and labels, SLAPS aims to reconstruct the graph that works best for the node classification task.

\subsubsection{Target GNN Model.} We employ a 2-layer graph convolution network (GCN) \cite{kipf2017semi} as our target GNN model. We used a learning rate of 0.001 and trained for 200 epochs.

\subsubsection{Evaluation Metrics} 
\label{sec:eval_metrics}
Following the existing works \cite{zhanggraphmi, he2021stealing}, we use the area under the receiver operating cost curve (AUC) and average precision (AP) to evaluate our attack. For all experiments (including baselines), we randomly sample an equal number of pairs of connected and unconnected nodes from the original graph to construct the test set. We measure the AUC and the AP on these randomly selected node pairs. We elaborate more in Appendix \ref{sec:data-sampling-elaborate}. All our experiments were conducted for $10$ different instantiations using PyTorch Geometric library \cite{fey2019fast} on 11GB GeForce GTX 1080 Ti GPU, and we report the mean values across all runs. The standard deviation of the results is in Appendix \ref{sec:main-result-std}.

\subsubsection{Datasets}
\label{sec:dataset}
We use three commonly used datasets chosen based on varying graph properties, such as their feature dimensions and structural properties. The task on all datasets is node classification. We present the data statistics in Table \ref{tab:data-stat}. We also performed additional experiments using \pubmed, \cseer, and \credit datasets (See Appendix \ref{sec:additional-datasets}).

\paragraph{\textbf{\cora}} The \cora dataset~\cite{sen2008collective} is a citation dataset where each research article is a node, and there exists an edge between two articles if one article cites the other. Each node has a label that shows the article category. The features of each node are represented by a 0/1-valued word vector, which indicates the word's presence or absence from the article's abstract.

\paragraph{\textbf{\coraml}} In \coraml dataset~\cite{sen2008collective}, each node is a research article, and its abstract is available as raw text. In contrast to the above dataset, the raw text of the abstract is transformed into a dense feature representation.
We preprocess the text by removing stop words, web page links, and special characters. Then, we generate Word2Vec~\cite{mikolov2013efficient} embedding of each word. Finally, we create the feature vector by taking the average over the embedding of all words in the abstract.

\paragraph{\textbf{\bitcoin}} The \bitcoin-Alpha dataset\cite{kumar2016edge} is a signed network of trading accounts. Each account is represented as a node, and there is a weighted edge between any two accounts, which represents the trust between accounts. The maximum weight value is +10, indicating total trust, and the lowest is -10, which means total distrust. Each node is assigned a label indicating whether or not the account is trustworthy. The feature vector of each node is based on the rating by other users, such as the average positive or negative rating. We follow the procedures in ~\cite{vu2020pgm} for generating the feature vectors.

\begin{table}
\caption{Dataset statistics. $|V|$ and $|E|$ denotes the number of nodes and edges respectively, $\mathcal{C}$, $\mathbf{X}_d$, and \textbf{deg} denotes the number of classes, size of feature dimension and the average degree of the corresponding graph dataset.} 
\label{tab:data-stat}
\begin{tabular}{lccc}
\toprule
& \textbf{\cora} & \textbf{\coraml} & \textbf{\bitcoin} \\\midrule
$|V|$    & 2708     & 2995        & 3783 \\
$|E|$   & 10858      & 8226    & 28248 \\
$\mathbf{X}_d$ & 1433   & 300    & 8 \\
$\mathcal{C}$  & 7                 & 7             & 2 \\
\textbf{deg}  & 4.00                 & 2.75             & 7.50 \\\bottomrule
\end{tabular}
\end{table}

\subsection{Result Analysis}

\begin{table*}[h!!]
\caption{Attack performance and baselines. The best performing attack(s) on each explanation method is(are) highlighted in bold, and the second best attack(s) is(are) underlined.}
\centering
\begin{tabular}{clp{0.7cm}<{\centering}p{0.7cm}<{\centering}p{0.7cm}<{\centering}p{0.7cm}<{\centering}p{0.7cm}<{\centering}p{0.7cm}<{\centering}p{0.7cm}<{\centering}p{0.7cm}<{\centering}p{0.7cm}<{\centering}p{0.7cm}<{\centering}p{0.7cm}<{\centering}p{0.7cm}<{\centering}p{0.7cm}<{\centering}p{0.7cm}<{\centering}p{0.7cm}<{\centering}p{0.7cm}<{\centering}}

\toprule
                \multicolumn{1}{l}{$Exp$} & 
                \multirow{1}{*}{\textbf{Attack}}& 
                \multicolumn{2}{c}{\textbf{\cora}} & \multicolumn{2}{c}{\textbf{\coraml}} & \multicolumn{2}{c}{\textbf{\bitcoin}}\\ 
                \cmidrule(r){3-4} \cmidrule(r){5-6} \cmidrule(r){7-8} 
                         &  & AUC           & AP            & AUC      & AP & AUC      & AP \\ \midrule

\parbox[t]{2mm}{\multirow{4}{*}{\rotatebox[origin=c]{90}{Baseline}}}
& \attrsim & 0.799	& 0.827 & 0.706 	& 0.753 & 0.535	& 0.478 \\
& \lsa \cite{he2021stealing} & 0.795	& 0.810 & 0.725	& 0.760 & 0.532	& 0.500\\
& \graphmi \cite{zhanggraphmi} & 0.856	&  0.830 & 0.808	& 0.814 & 0.585	& 0.518 \\ 
& \gsl \cite{fatemi2021slaps} & 0.736	& 0.776 & 0.649	& 0.702 & 0.597	& 0.577\\ \midrule

\parbox[t]{2mm}{\multirow{5}{*}{\rotatebox[origin=c]{90}{\grad}}} 
& \faac & 0.734	& 0.773 & 0.640	& 0.705 & 0.527	& 0.515 \\
& \faaw & 0.678	& 0.737 & 0.666	& 0.730 & 0.264 	& 0.383 \\
& \exhoney & \underline{0.948}	& \underline{0.953} & \textbf{0.902}	& \underline{0.833} & \textbf{0.700}	& \textbf{0.715}\\ 
& \exohoneyex & 0.924 	& 0.939 & 0.699 	& 0.768  & 0.229	& 0.365\\
& \exoattrsim & \textbf{0.984} & \textbf{0.978} & \underline{0.890}  & \textbf{0.891}	& \underline{0.681}  & \underline{0.644} \\ \midrule

\parbox[t]{2mm}{\multirow{5}{*}{\rotatebox[origin=c]{90}{\gradinput}}} 
& \faac & 0.734          & 0.775          & 0.674	& 0.734 & 0.525	& 0.527\\ 
& \faaw & 0.691 & 0.742 & 0.717 & 0.756 & 0.252	& 0.380\\
& \exhoney & \underline{0.949}	& \underline{0.950} & \underline{0.887}	& \underline{0.832} & \textbf{0.709}	& \textbf{0.723}\\ 
& \exohoneyex &0.903	& 0.923 &0.717	& 0.781 & 0.256	& 0.380\\
& \exoattrsim &\bf{0.984}	& \bf{0.979} & \bf{0.903}	& \bf{0.899} & \underline{0.681}	& \underline{0.644}\\ \midrule

\parbox[t]{2mm}{\multirow{5}{*}{\rotatebox[origin=c]{90}{\zorro}}} 
& \faac & 0.823	& 0.860          & 0.735	& 0.786 & \underline{0.575} & 0.529\\ 
& \faaw & 0.723	& 0.756 & 0.681	& 0.697 & 0.399 & 0.449\\
& \exhoney & \bf 0.884	& \bf 0.880 & \underline{0.776}	& \underline{0.820} & 0.537 & \underline{0.527}\\ 
& \exohoneyex & 0.779	& 0.810 & 0.722	& 0.777 & \bf 0.596 & \bf 0.561\\
& \exoattrsim & \underline{0.871}	& \underline{0.873} & \bf 0.806 & \bf 0.829 & 0.427 & 0.485\\ \midrule

\parbox[t]{2mm}{\multirow{5}{*}{\rotatebox[origin=c]{90}{\szorro}}} 
& \faac & 0.907          & 0.922          & \underline{0.747}	& \underline{0.791} & \bf{0.601} & \bf{0.590}\\ 
& \faaw & 0.794 & 0.815 & 0.712	& 0.740 & 0.490 & 0.491\\
& \exhoney & \bf{0.918}	& \underline{0.923} & \textbf{0.776}	& \textbf{0.819} & \underline{0.598} & \underline{0.565}\\ 
& \exohoneyex & 0.893	& 0.915 & 0.742	& 0.784 & 0.571 & 0.564\\
& \exoattrsim & \underline{0.908}	& \bf{0.934} & 0.732	& 0.787 & 0.484 & 0.496\\ \midrule

\parbox[t]{2mm}{\multirow{5}{*}{\rotatebox[origin=c]{90}{\glime}}} 
& \faac & \underline{0.643}          & \underline{0.710}          & \underline{0.610}	& \underline{0.652} & \underline{0.473} & \underline{0.493}\\ 
& \faaw & 0.516 & 0.522 & 0.517	& 0.528 & 0.264 & 0.371\\
& \exhoney & \bf{0.730}	& \bf{0.773} & \bf{0.681}	& \bf{0.740} & \bf{0.542} & \bf{0.525}\\ 
& \exohoneyex &0.558	& 0.571 & 0.540	& 0.555 & 0.236 & 0.361\\
& \exoattrsim &0.505	& 0.524 & 0.520	& 0.523 & 0.504 & 0.512\\ \midrule

\parbox[t]{2mm}{\multirow{5}{*}{\rotatebox[origin=c]{90}{\gnnexp}}} 
& \faac & 0.614          & 0.650          & 0.653	& 0.705 & 0.467 & 0.489\\ 
& \faaw & \underline{0.724} & \underline{0.760} & \underline{0.637}	& \underline{0.692} & 0.390 & 0.454\\
& \exhoney & \bf{0.762}	& \bf{0.796} & \bf{0.700}	& \bf{0.695} & \bf{0.590} & \bf{0.563}\\ 
& \exohoneyex &0.517	& 0.552 & 0.490	& 0.508 & 0.386 & 0.451\\
& \exoattrsim &0.537	& 0.541 & 0.484	& 0.508 & \underline{0.551} & \underline{0.543}\\ \bottomrule

\end{tabular}
\quad
\begin{tabular}{p{7.5cm}}
\begin{tcolorbox}
\underline{\textsc{\Large Summary of attack comparisons}}
\begin{itemize}[leftmargin=*]  \setlength\itemsep{1em}
\large    
\item The amount of information in the explanation alone for the graph structure can be quantified using the \exoattrsim attack.
 \item Explanation only (\exoattrsim) and explanation augmentation (\exhoney) attacks for all explanation methods other than \glime and \gnnexp outperform all baseline methods.
\item Among the baseline approaches, the white-box access-based attack, \graphmi, performs the best, followed by \attrsim.
    \item The relatively good performance of \attrsim in datasets points to a high correlation of node features with node connections in all datasets other than \bitcoin.
    \item The information leakage for \bitcoin is limited by small feature size. Our results on additional datasets with varying feature dimensions, as discussed in Section \ref{sec:additional-experiments-results}, support our observation about the effect of small feature size on the attack's success.
    \item For \glime and \gnnexp, we observe that the explanation contains little information about the graph structure. The reason behind this is further revealed in the fidelity-sparsity analysis of the obtained explanations.
   
\end{itemize}
\end{tcolorbox}

\end{tabular}
\label{tab:main-result}
\end{table*}

\subsubsection{\textbf{\large Analysing the Information Leakage by Explanations (RQ~\ref{rq:attack-performance})}}
\label{sec:attack-performance}
The detailed results of different attacks are provided in Table \ref{tab:main-result}.
Our results show that the explanation-only (\exoattrsim) and explanation augmentation (\exhoney) attacks for all explanation methods other than \glime and \gnnexp outperform all baseline methods and by far reveal the most information about the private graph structure. We attribute the superior performance of \exhoney to the multi-task learning paradigm that aims to reconstruct both the features and explanations. Our results also support our assumption that a graph structure that is good for predicting the node labels is also suitable for predicting the node features and explanations. 
\begin{figure*}[h!]
\centering
\subfigure[\cora]{\label{fig:cora-attrsim-exp-only}\includegraphics[width=0.3\linewidth]{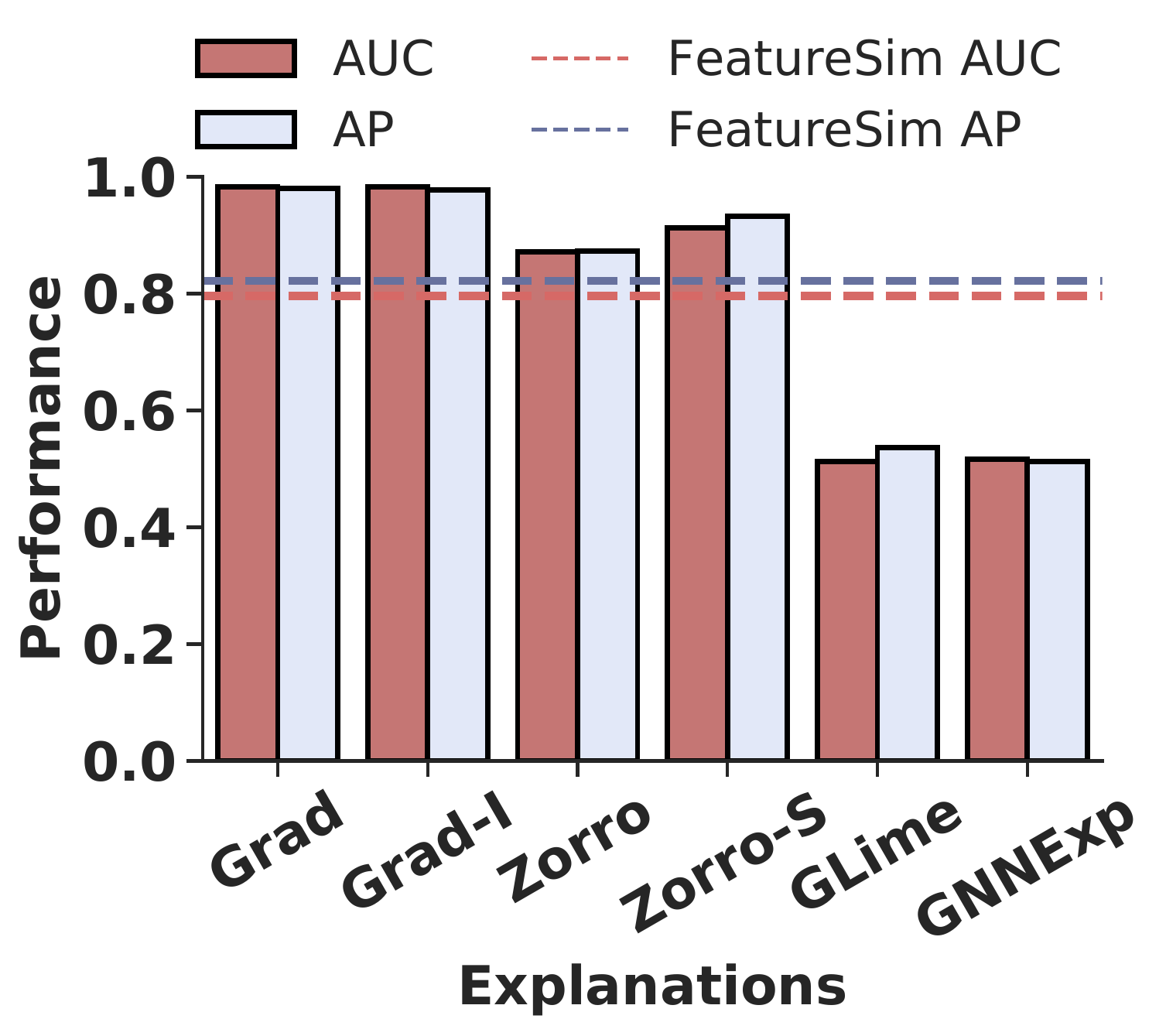}}
\subfigure[\coraml]{\label{fig:coraml-attrsim-exp-only}\includegraphics[width=0.3\linewidth]{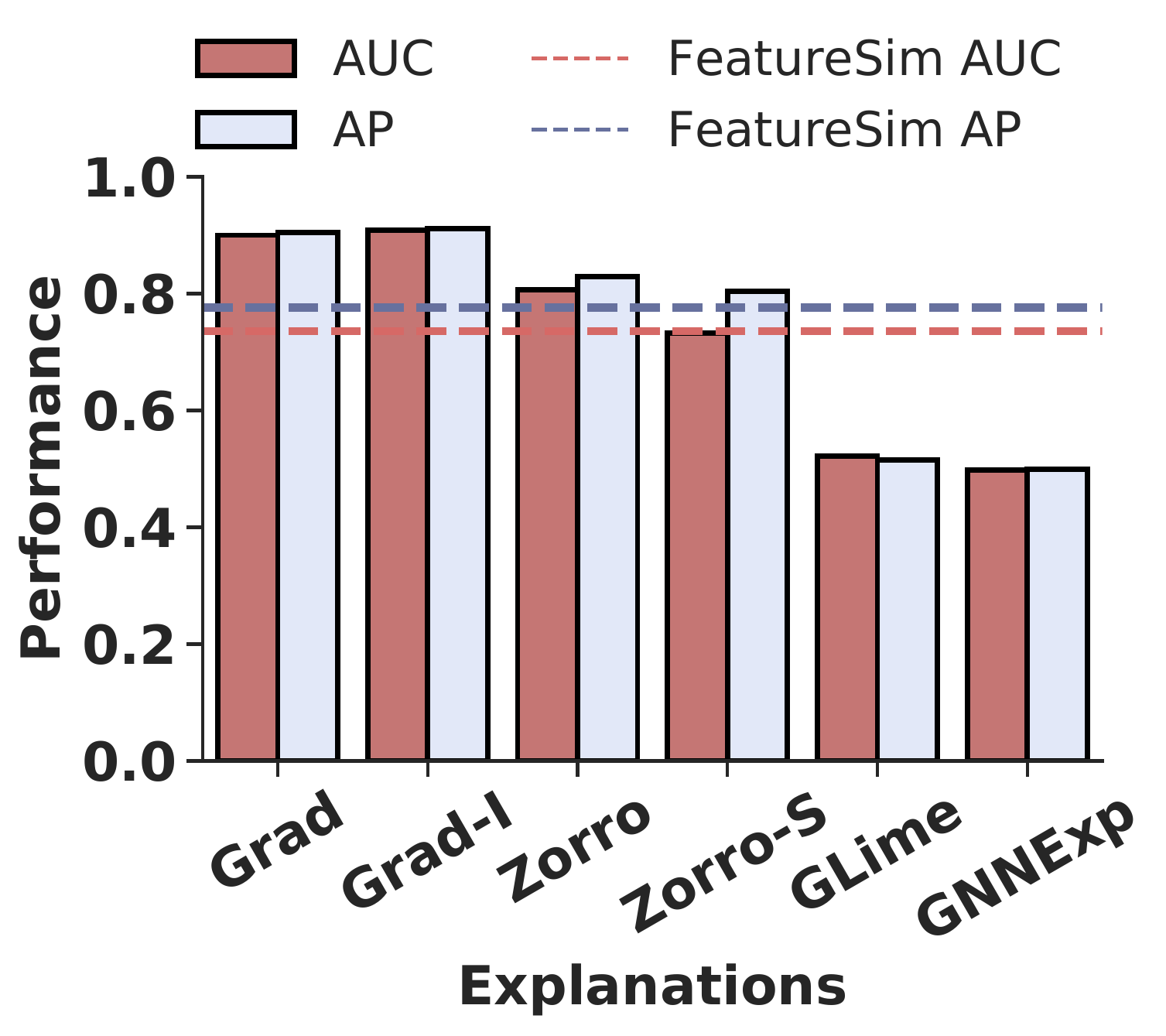}}
\subfigure[\bitcoin]{\label{fig:bitcoin-attrsim-exp-only}\includegraphics[width=0.3\linewidth]{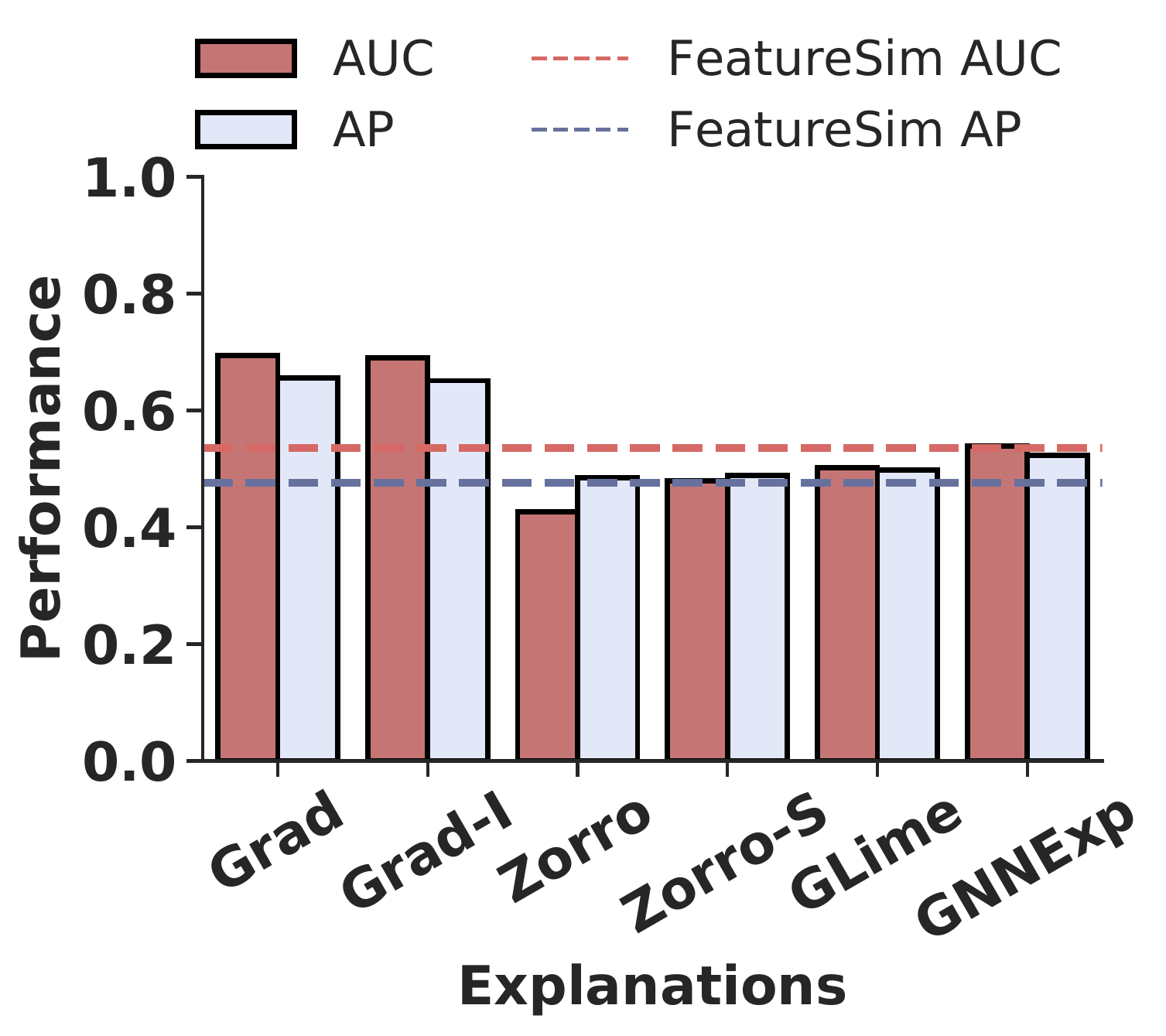}}

\caption{Performance of explanation-only attack (\exoattrsim) on the different datasets. The adopted baseline is \attrsim which performs the pairwise similarities using the true node features.}
\label{fig:attrsim-ex}
\end{figure*}

In the following, we provide an in-depth result analysis.

\begin{itemize}[leftmargin=*]
\item[$\circ$] \textbf{Baseline Methods.} Among the baseline methods, \graphmi is the best-performing attack. This is not very surprising as \graphmi has white-box access to the target GNN in addition to access to node features and labels. On the contrary, the \attrsim attack, which only uses node features, shows competitive performance. This highlights the fact that the features alone are very informative of the graph structure of the studied datasets (except \bitcoin in which all baseline attacks almost fail with an AUC score of close to 0.5).
    \item[$\circ$] \mpara{Comparison of the Privacy Leakage via Explanations With That of Features.} Figure \ref{fig:attrsim-ex} compares the performance of two similarity-based attacks using explanations (\exoattrsim)  and features (\attrsim), respectively. Note that these attacks do not use any other information except explanations and features, respectively. Hence, allowing us to compare their information content. We note that except for \glime and \gnnexp, \exoattrsim outperforms \attrsim for all datasets except \bitcoin. Moreover, for \bitcoin, both attacks fail (with AUC close to 0.5) except for gradient-based explanation methods. 
      \item[$\circ$] \mpara{Explanation Augmentation With Features and Labels.} Next, we compare the explanation augmentation attack \exhoney with the vanilla graph structure learning approach \gsl, which only uses node features and labels. \exhoney outperforms \gsl (Figure \ref{fig:ex-honey}), which points to the added utility of using explanations to reconstruct the graph structure.
      
         \item[$\circ$] \mpara{Explanation Augmentation Attack Variants.} 
         In Figure \ref{fig:ex-honey-ex}, we compare the performance of \exohoneyex, which have no access to the true features, with \gsl, which utilizes the true features. We observe that on most datasets, the attack significantly outperforms \gsl. This emphasizes that explanations encode the feature information albeit the importance.
         Comparing the performance of \faac and \faaw with \exhoney on all datasets shows that independently extracting the adjacency matrix from the features and explanations respectively and then combining the two adjacency matrices is better than combining the features and explanations at the input stage. 
         
            \item[$\circ$] \mpara{Attack on \bitcoin.} We observe that all baseline attacks fail for \bitcoin. We attribute this to the very small feature size (=8) and the small number of labels (=2). The attacks are not able to exploit the little information exposed by a small set of features/labels. Explanation-based attacks, especially in the case of gradient-based explanations, are more successful than baselines but less than for other datasets.
       
      \end{itemize}
      
\subsubsection{Results on Additional Datasets}
\label{sec:additional-experiments-results}
We further investigate our observation of the impact of feature dimension and disasno on the attack success by performing experiments on the additional datasets (namely \cseer, \pubmed, and \credit) with varying graph properties (details on datasets are provided in Section \ref{sec:additional-datasets}). The detailed result is in the Appendix in Table \ref{tab:additional-exp}.
We observe that the result on \cseer and \pubmed supports our observation of the effect of the feature size on the attack success on \cora and \coraml, respectively. 
First, a dataset with a high feature dimension is most vulnerable to the attack, as seen in the relative performance of \cseer on all explanation methods and attacks. Secondly, on the \grad and \gradinput explanations, simply augmenting backpropagation-based explanations with the actual feature via concatenation (\faac) or element-wise multiplication (\faaw) does not encode the necessary information for launching a successful attack. Instead, it introduces new noise to the reconstructed graph because the explanations are generated such that larger gradients, which only reflect the sensitivity between input and output, may not accurately indicate the importance of the feature. We also observe this phenomenon on other explanations methods, including \zorro, \szorro, and \glime, where \exohoneyex, which uses only explanations as input, performs better than \faac and \faaw on the \cora, \coraml, \cseer, and \pubmed datasets.
Thirdly, explanation methods that are robust to input perturbation, such as \zorro and \szorro, achieve better attack performance than backpropagation-based explanations on all explanation augmentation attacks. Finally, \glime and \gnnexp explanations are more robust to the attacks except on \exhoney attack, which achieves a high AUC and AP scores.

Among all the additional datasets, the results of the baseline attacks (\attrsim, \lsa, \graphmi, and \gsl, which all rely only on feature information) on the \credit dataset are the least. Recall that the \credit dataset has only 13 features. This supports our observation that small feature size affects the attack's success. Notably, the best performing baseline (\graphmi), which utilizes white-box access to the trained model, has a very low AUROC of $0.59$ and $0.51$ on the \bitcoin and \credit datasets, respectively. This is in contrast to the performance of \graphmi on \cora, \cseer, \pubmed, and \coraml with AUROC of $0.86$, $0.79$, $0.80$, and $0.81$, respectively.

Furthermore, in comparison with \bitcoin, which has only 8 features, we observe that the performance of \attrsim on the \credit dataset is 34\% higher ($0.54$ on \bitcoin and $0.72$ on \credit), which points to high connectivity among similar nodes on the \credit dataset. This, in turn, boosts our attack performance and allows better reconstruction of the private graph. Therefore, the poor performance of our attacks on the \bitcoin dataset is due to disassociativity on the \bitcoin dataset. Moreover, it is worth noting that our explanation augmentation attack is highly successful on the \credit dataset.
\subsubsection {\textbf{\large Differences in Privacy Leakage (RQ \ref{rq:exp-vulnerability}).}}
All explanation method leaks significant information via the reconstruction attacks except for \glime and \gnnexp. We observe that for \glime and \gnnexp, the explanation-based attacks do not perform better than the baselines, which do not utilize any explanation. Moreover, gradient-based methods are most vulnerable to privacy leakage.
To understand the reason behind these observations, we investigate the  explanation quality. We measure the goodness of the explanation by its ability to approximate the model's behavior which is also referred to as \textit{faithfulness}. As the groundtruth for explanations is not available, we use the RDT-Fidelity proposed by \cite{funke2021zorro} to measure faithfulness. The results are shown in Table \ref{tab:rdt-fidelity-sparsity}. We further note that by definition, the complete input is faithful to the model. Therefore, in addition, we measure the sparsity of the explanation. A meaningful explanation should be sparse and only contain a small subset of features most predictive of the model decision. We use the entropy-based sparsity definition from \cite{funke2021zorro} as it is applicable for both soft and hard explanation masks. The results are shown in Table \ref{tab:rdt-fidelity-sparsity}. We analyze the tradeoffs of privacy and explanation goodness in the following.
   \begin{itemize}[leftmargin=*]
       \item \textbf{First}, we observe that 
       \gnnexp has the lowest sparsity (the higher the entropy, the more uniform the explanation mask distribution, i.e., higher the explanation density). In other words, almost all features are marked equally important. Hence, it is not surprising that it shows a high fidelity.
       This is the main reason why \exoattrsim fails because there is no distinguishing power contained in the explanations. 
       \item \textbf{Second}, we observe that gradient-based explanations (\grad and \gradinput) contain the most information about the graph structure though they have low fidelity, i.e., they do not reflect the model's decision process. It appears that these two methods provide the most similar explanations for connected nodes. This is really the worst case when the explanations are not useful but leak maximum private information about the graph structure.
       \item \textbf{Third}, \glime has the highest sparsity and lowest fidelity. \glime runs the HSIC-Lasso feature selection method over a local dataset created using the node and its neighborhood. HSIC-Lasso is known to output a very small set of most predictive features \cite{dong2020featureselection} when used for global feature selection.   
       But for the current setting of instance-wise feature selection, i.e., finding the most predictive features for decision over an instance/node, \glime's explanation turns out to be too short, which is neither faithful to the model nor contains any predictive information about the neighborhood.
         \item \textbf{Finally}, the explanations of \zorro and \szorro show the highest fidelity and intermediate sparsity, pointing to their high quality. The \exhoney attack also obtains high AUC scores for two datasets pointing to the expected increased privacy risk with an increase in explanation utility.  
   \end{itemize}

\begin{table}[h!!]
\caption{RDT-Fidelity and sparsity (entropy) of different explanation methods. For fidelity, the higher the better. For sparsity, the lower the better}
\label{tab:rdt-fidelity-sparsity}
\begin{tabular}{lp{0.7cm}<{\centering}p{0.7cm}<{\centering}p{0.7cm}<{\centering}p{0.7cm}<{\centering}p{0.7cm}<{\centering}p{0.7cm}<{\centering}} \toprule
$Exp$    & \multicolumn{2}{c}{\cora} & \multicolumn{2}{c}{\coraml} & \multicolumn{2}{c}{\bitcoin} \\ 
\cmidrule(r){2-3} \cmidrule(r){4-5} \cmidrule(r){6-7}
       & \small{Fidelity}    & \small{Sparsity}   & \small{Fidelity}     & \small{Sparsity}    & \small{Fidelity}     & \small{Sparsity}     \\ \midrule
\textbf{\grad}   & 0.23         & 3.99        & 0.22          & 5.24         & 0.83          & 0.64          \\
\textbf{\gradinput} & 0.19         & 3.99        & 0.20          & 5.30         & 0.82          & 0.64 \\
\textbf{\zorro} & 0.89         & 1.83        & 0.96          & 3.33         & 0.99          & 0.37 \\
\textbf{\szorro} & 0.98         & 2.49        & 0.84          & 2.75         & 0.95          & 0.96 \\
\textbf{\glime} & 0.19         & 0.88        & 0.20          & 0.98         & 0.82          & 0.13 \\
\textbf{\gnnexp} & 0.74         & 7.27        & 0.55          & 5.70         & 0.90          & 2.05 \\ \bottomrule
\end{tabular}
\end{table}

\begin{figure*}
\centering
\subfigure[\cora]{\label{fig:cora-exp-only}\includegraphics[width=0.28\linewidth]{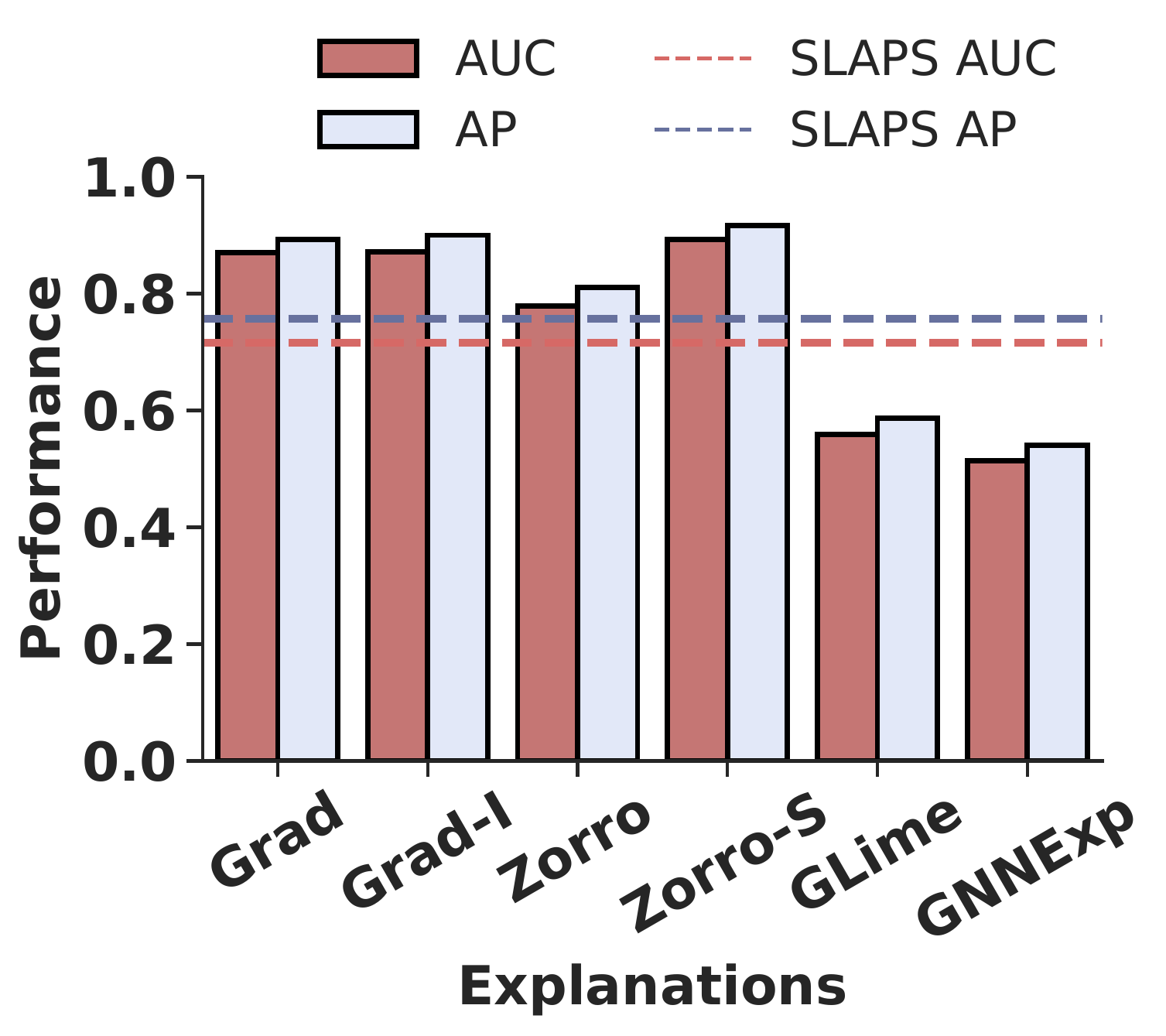}}
\subfigure[\coraml]{\label{fig:coraml-exp-only}\includegraphics[width=0.28\linewidth]{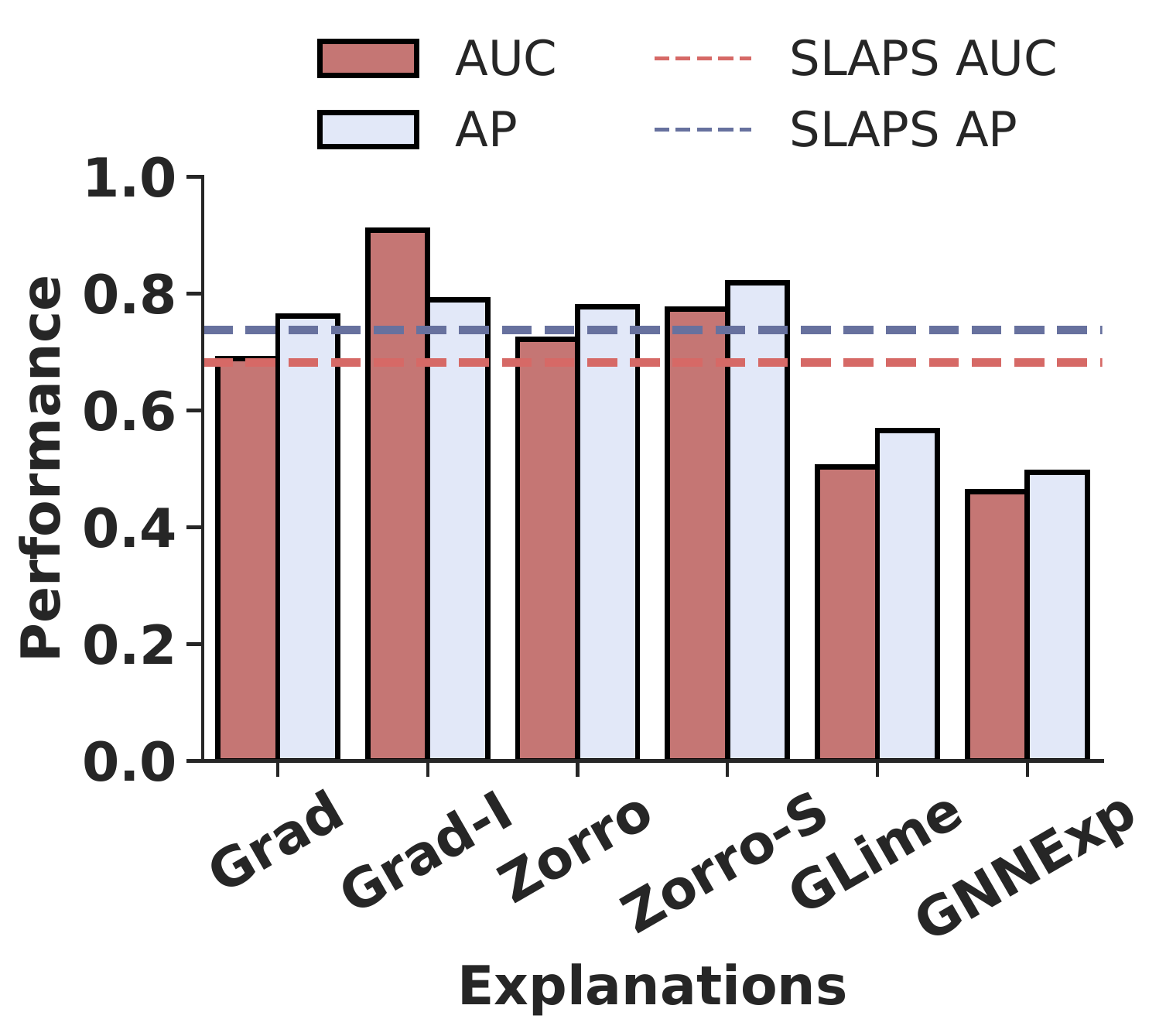}}
\subfigure[\bitcoin]{\label{fig:bitcoin-exp-only}\includegraphics[width=0.28\linewidth]{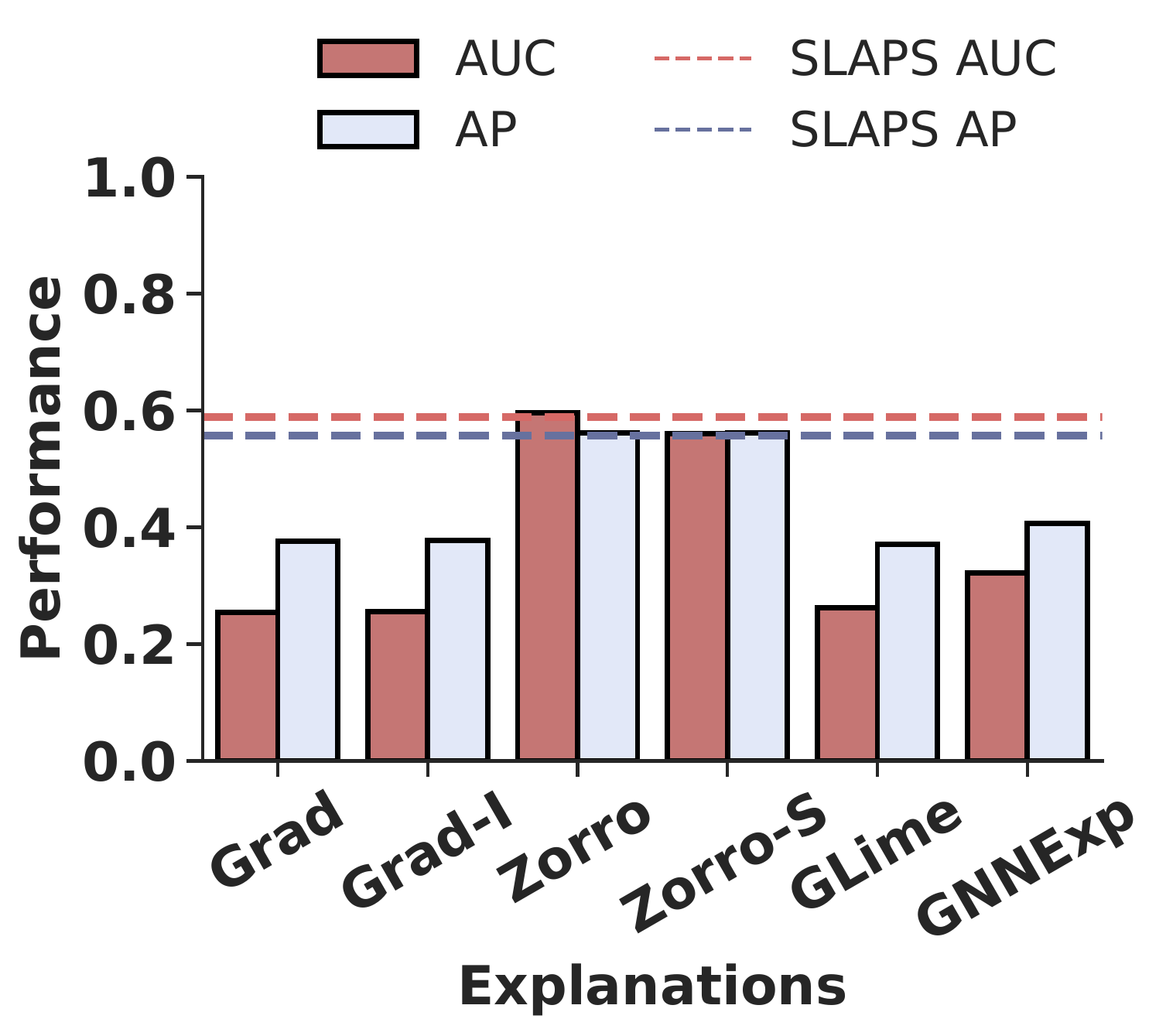}}

\caption{Average AUC and AP  of \exohoneyex attack on the different datasets. The adopted baseline is \gsl which use the true node features.}
\label{fig:ex-honey-ex}
\end{figure*}

\begin{figure*}
\centering
\subfigure[\cora]{\label{fig:cora-gsef-slaps}\includegraphics[width=0.28\linewidth]{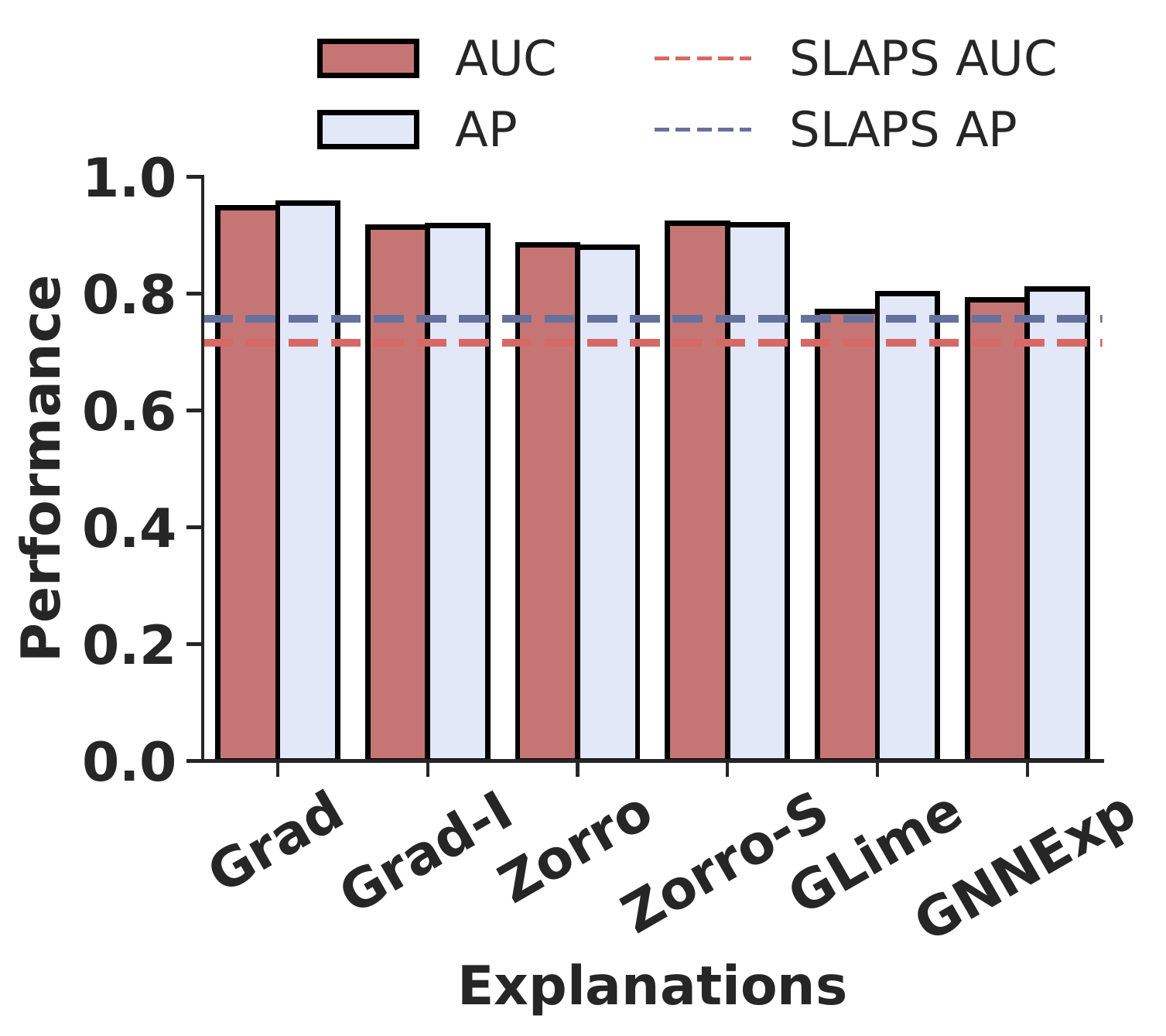}}
\subfigure[\coraml]{\label{fig:coraml-gsef-slaps}\includegraphics[width=0.28\linewidth]{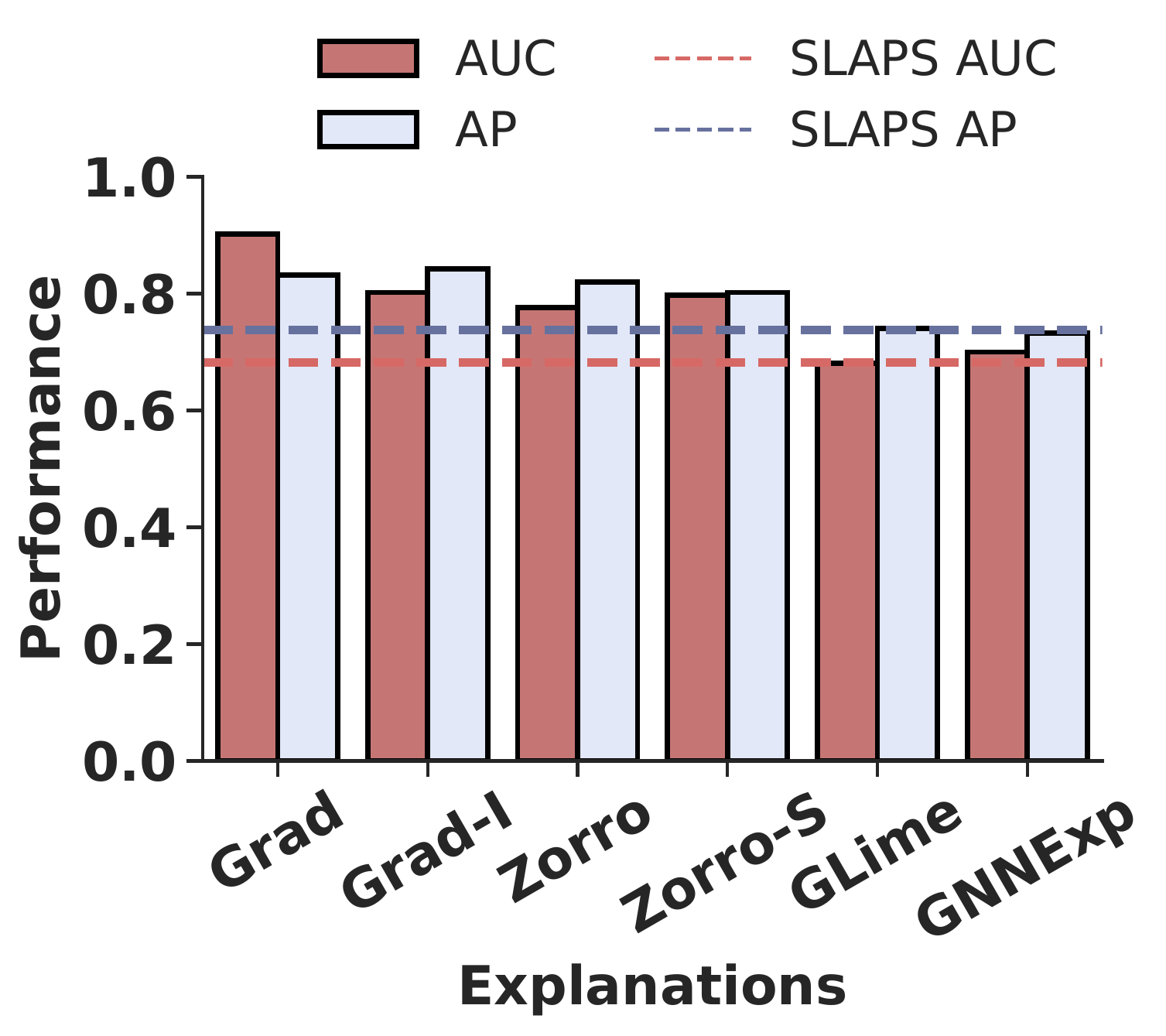}}
\subfigure[\bitcoin]{\label{fig:bitcoin-gsef-slaps}\includegraphics[width=0.28\linewidth]{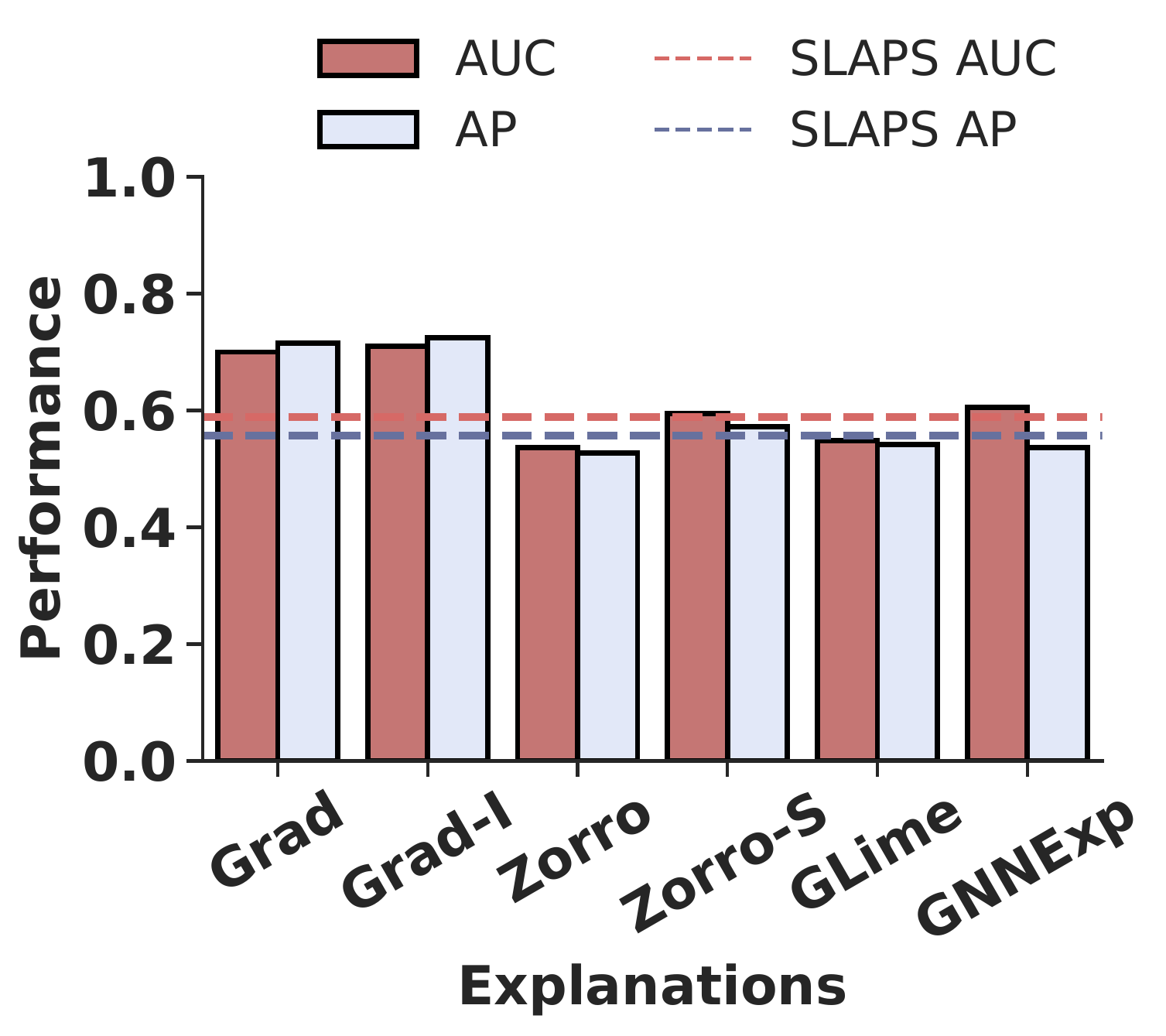}}

\caption{Average AUC and AP  of \exhoney attack on the different datasets. The adopted baseline is \gsl which use the true node features.}
\label{fig:ex-honey}
\end{figure*}

\subsubsection{\textbf{\large Adversary's Advantage in Terms of Trained GNN Model (RQ~\ref{rq:reconstruct-edges})}}
\label{sec:reconstruct-edges}
Here, we formalize a quantitative advantage measure that captures the privacy risk posed by the different attacks. The attacker is at an advantage if she can train a well-performing model (on a downstream task) using the reconstructed graph. As the attack models based on graph structure learning implicitly train a GNN on the reconstructed graph, we quantify the attacker's advantage by the performance on the downstream task of node classification.

\begin{hypothesis}
If the explanations and the reconstructed graph can perform better on a downstream classification task with high confidence, then the reconstructed adjacency is a valid representation of the graph structure. Hence, the attacker has an advantage quantified by Equation \ref{eq:attacker_advantage}.
\end{hypothesis}
We define the attacker's advantage as 
\begin{gather}
\label{eq:attacker_advantage}
    Advantage = \mathcal{R}(f(\mathcal{I}_X;Adj_{rec};\theta_W), y),
\end{gather}
where $f$ is a 2-layer GCN model parameterized by $\theta_W$, $\mathcal{I}_X$ is the input matrix which can either be only the explanations or the combination of the explanations and features, $Adj_{rec}$ is the reconstructed graph by the attacker, $y$ is the groundtruth label and $\mathcal{R}$ is an advantage measure that compares the predictions on $f$  and the groundtruth label. We use accuracy as the choice of $\mathcal{R}$.

We compare the results to that of \gsl that uses the actual features and groundtruth label for learning the graph structure and the original performance (denoted by $Max$) in which the model is trained with true features, labels, and graph structure.
We analyze the attacker's advantage corresponding to four attacks; \\\faac, \faaw, \exhoney, and \exohoneyex.
The intuition is that if the attacker's advantage is not better than \gsl, then the best advantage an attacker can have is similar to having the actual feature and performing graph structure learning. Also, if the attacker's advantage is greater or equal to $Max$, then the attacker has an equivalent advantage as she would have by possessing the actual feature and graph. An example use case of the attacker's advantage is shown in Figure \ref{fig:advantage_intuition}. Specifically, if a model trained with, say, Jane's full data (true features and graph) and another trained only with her explanations (no graph or true features), both models will make the same prediction about Jane.
\begin{figure}
    \centering
    \includegraphics[width=0.47\textwidth]{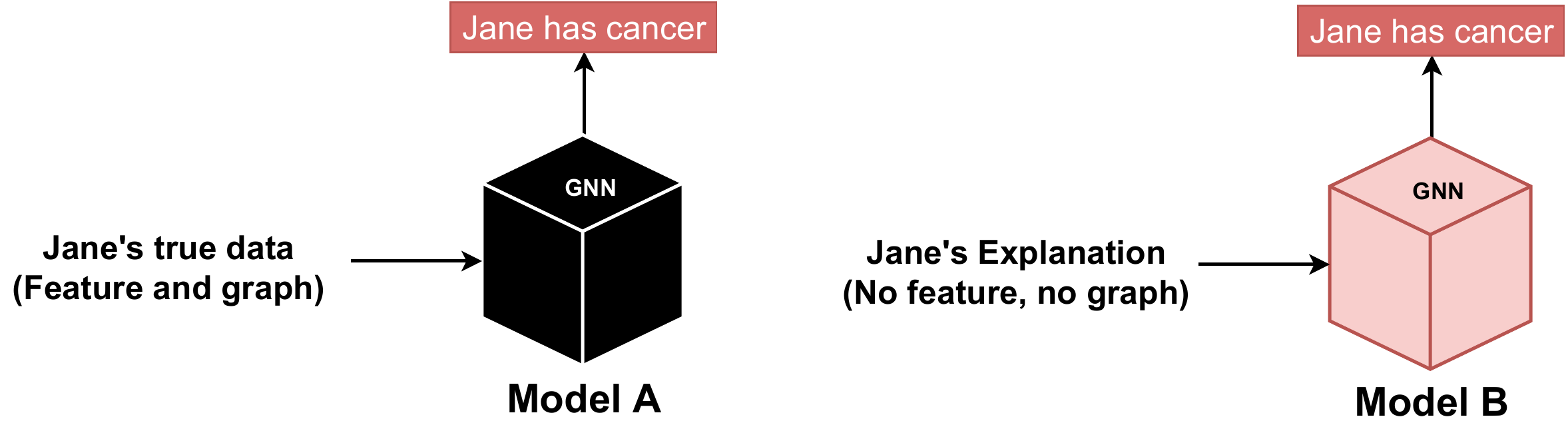}
    \caption{An example use case of the attacker's advantage}
    \label{fig:advantage_intuition}
\end{figure}
The detailed results for the attacker's advantage are plotted in Figure~\ref{fig:advantage}. We observe that on \cora, the attacker obtains the highest advantage for \grad, \gradinput, \zorro, and \szorro explanations. 
On \coraml, the highest advantage is obtained with \grad and \szorro explanations. Usually, the attacker's advantage is positively correlated with the success rate of corresponding attacks for both \cora and \coraml.
On \bitcoin, the attacker's advantage for all explanation methods is usually high. This is surprising as the success rate for attacks is relatively lower than for other datasets. This might imply that the reconstructed graph has the same semantics as the true graph, if not the exact structure.

\begin{figure*}
\centering 
\subfigure[\cora]{\label{fig:cora-reconstructed-acc}\includegraphics[width=1\linewidth]{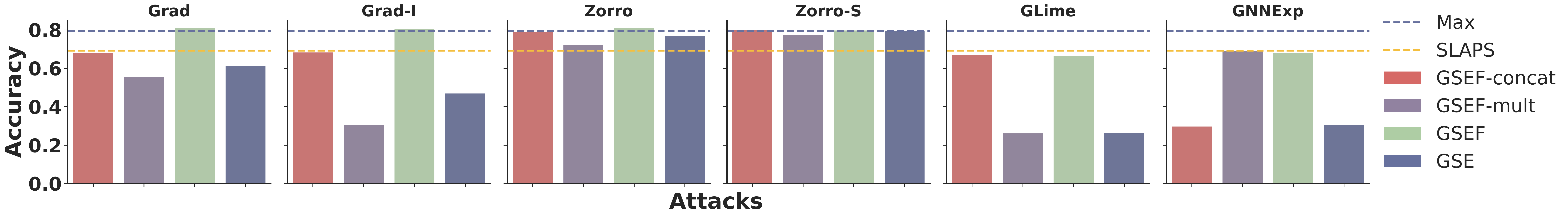}}
\subfigure[\coraml]{\label{fig:coraml-reconstructed-acc}\includegraphics[width=1\linewidth]{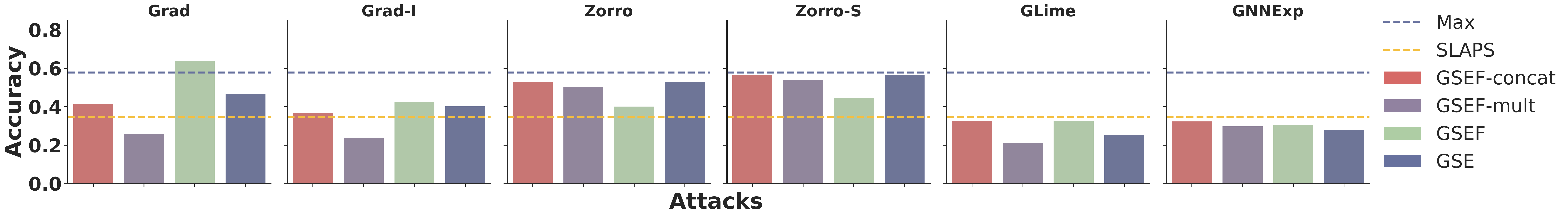}}
\subfigure[\bitcoin]{\label{fig:bitcoin-reconstructed-acc}\includegraphics[width=1\linewidth]{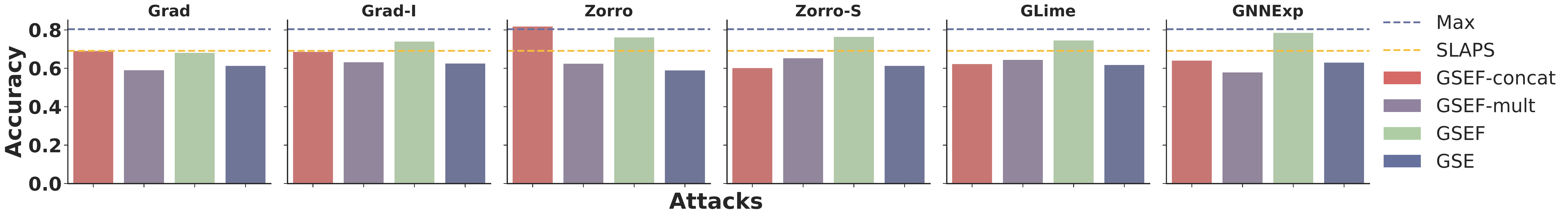}}

\caption{Accuracy of the reconstructed graph on a downstream node classification task by all models on different datasets. The blue line is the original accuracy using the true features and edges, while the yellow line is the \gsl accuracy.}
\label{fig:advantage}
\end{figure*}

\subsubsection{\textbf{\large Lack of Groundtruth Labels (RQ~\ref{rq:targetmodelaccess})}}
\label{sec:targetmodelaccess}
We relax the assumption that the attacker has access to groundtruth labels. Instead, she has black-box access to the target model. This is made possible with the popularity of machine learning as a service (MLaaS), where a user can input a query and get the predictions as output. Therefore, the "groundtruth" label is the one obtained from the target model. As representative explanations, we show the performance on \grad and \zorro on all datasets.

As shown in Table \ref{tab:groundtruth-vs-blackbox-attack}, on the \grad explanation, we observe a 3\% gain in attack performance on \faac when the attacker has access to the target model on the \cora and \coraml dataset. On the \bitcoin dataset, we observe a decrease of about 3\% in AUC and an increase of 6\% in AP. The corresponding performance on \zorro follows the same with no significant change in the performance on \coraml and a 5\% decrease in performance on \cora.

The performance of \faaw and \exhoney attack decreases across all datasets, with \bitcoin having the worst performance reduction of up to 29\% on \grad. However, on \zorro, there is a 2\% gain on \cora, a 4\% decrease on \coraml, and up to 36\% decrease on \bitcoin. We observe performance drop on \exhoney and \exohoneyex on all datasets across both explanations except for \exohoneyex on \grad, which  has up to 5\% gain in performance on the \cora and \coraml datasets.
It is important to note that the \bitcoin dataset has the least performance across all attacks, even when the true groundtruth label is used. Therefore, the large disparity in performance when the label is generated from the trained black-box model is not surprising. 

\textbf{Summary.} In the absence of groundtruth labels, we used model predictions. Intuitively, for incorrectly predicted nodes, the model explanations which are based on incorrectly predicted labels should lead to lower performance in the presence of groundtruth labels. For \cora and \coraml datasets, on \faac and \exohoneyex attacks, having black-box access to the target model performs better than the attacker having access to the groundtruth label. For \faaw and \exhoney attacks, it is better to have access to groundtruth label to achieve the best attack success rate. For the \bitcoin dataset, the groundtruth labels perform better than the black-box access to the target model on all attacks.

\begin{table*}[h!!]
\caption{Performance comparison of relaxing the availability of groundtruth labels (\textbf{$Y$})  assumption. Here, the attacker has black-box access to the target model ($\mathcal{M}$). We perform the experiment on all datasets' \grad and \zorro explanation methods. $\Delta$ is the percentage difference. A negative value implies that the groudtruth labels are preferred over black-box access.}
\centering
\begin{tabular}{p{2cm}lp{0.7cm}<{\centering}p{0.7cm}<{\centering}p{0.7cm}<{\centering}p{0.7cm}<{\centering}p{0.7cm}<{\centering}p{0.7cm}<{\centering} |p{0.7cm}<{\centering}p{0.7cm}<{\centering}p{0.7cm}<{\centering}p{0.7cm}<{\centering}p{0.7cm}<{\centering}p{0.7cm}<{\centering}}
\toprule
                \multicolumn{1}{l}{\textbf{Dataset}} & 
                \multirow{1}{*}{\textbf{Attack}}& 
                \multicolumn{2}{c}{\textbf{Y$_\grad$}} & \multicolumn{2}{c}{\textbf{$\mathcal{M_\grad}$}} & \multicolumn{2}{c}{\textbf{$\Delta_\grad$}}&
                \multicolumn{2}{c}{\textbf{Y$_\zorro$}} & \multicolumn{2}{c}{\textbf{$\mathcal{M_\zorro}$}} & \multicolumn{2}{c}{\textbf{$\Delta_\zorro$}}\\ 
                \cmidrule(r){3-4} \cmidrule(r){5-6} \cmidrule(r){7-8} \cmidrule(r){9-10} \cmidrule(r){11-12} \cmidrule(r){13-14} 
                         &  & AUC           & AP            & AUC      & AP & AUC      & AP 
                         & AUC           & AP            & AUC      & AP & AUC      & AP \\ \midrule

\parbox[t]{2mm}{\multirow{5}{*}{{\cora }}}
& \faac & 0.734	& 0.773    & 0.757	& 0.784 & 3.1	& 1.4
& 0.823	& 0.860    & 0.779	& 0.810 & -5.3	& -5.8\\
& \faaw &0.678	& 0.737 & 0.658	& 0.700 & -2.9	& -5.0
&0.723	& 0.756 & 0.740	& 0.772 & 2.4	& 2.0\\
& \exhoney &0.948	& 0.953  & 0.927	& 0.932 & -2.2 & -2.2
&0.884	& 0.880  & 0.871	& 0.881 & -1.5 & 0.1\\ 
& \exohoneyex &0.924	& 0.939  & 0.946	& 0.966 & 2.4	& 2.9
 &0.779	& 0.810  & 0.814	& 0.849 & 4.5	& 4.8 \\ \midrule

\parbox[t]{2mm}{\multirow{5}{*}{\shortstack[l]{\coraml }}} 
& \faac   & 0.640	& 0.705 & 0.661	& 0.734 & 3.3	& 4.1
& 0.735	& 0.786 & 0.738	& 0.792 & 0.4	& 0.8\\
& \faaw  & 0.666	& 0.730 & 0.650	& 0.693 & -2.4	& -5.1
& 0.681	& 0.697 & 0.653	& 0.692 & -4.1	& -0.7\\
& \exhoney & 0.902	& 0.833 & 0.808	& 0.852 & -10.4	& 2.3
& 0.776	& 0.820 & 0.751	& 0.796 & -3.2	& -2.9\\ 
& \exohoneyex & 0.699	& 0.768 & 0.735	& 0.795 & 5.2	& 3.5
 & 0.722	& 0.777 & 0.713	& 0.759 & -1.2	& -2.3\\ \midrule

\parbox[t]{2mm}{\multirow{5}{*}{\shortstack[l]{\bitcoin }}} 
& \faac & 0.527	& 0.515 & 0.512 & 0.546 & -2.7	& 6.0 
& 0.575	& 0.529 & 0.523 & 0.517 & -9.0	& -2.3\\
& \faaw & 0.264	& 0.383 & 0.241	& 0.367 & -8.7	& -4.2 
& 0.399	& 0.449 & 0.255	& 0.369 & -36.1	& -17.8 \\
& \exhoney & 0.700 & 0.715 & 0.500	& 0.560 & -28.6	& -21.7 
& 0.537 & 0.527 & 0.352	& 0.438 & -34.5	& -16.9\\ 
& \exohoneyex & 0.229	& 0.365 & 0.210	& 0.352 & -8.3	& -3.6 
& 0.596	& 0.561 & 0.491	& 0.503 & -17.6	& -10.3\\ \bottomrule

\end{tabular}
\label{tab:groundtruth-vs-blackbox-attack}
\end{table*}

\subsection{Attack Performance When Explanations for Partial Node-set Are Available}
\label{sec:exp-subset}
In this section, we assume that the attacker has access to explanations of only a subset of the nodes. The attacker is interested in reconstructing the subgraph induced on the partial node set. We perform the experiment on representative datasets \cora and \credit. For each of the datasets, we only use 30\% of the nodes to obtain explanations. The result is shown in Table \ref{tab:exp-subset}.

Unsurprisingly, we observe a drop in attack performance scores when explanations for only a node subset are available. However, we observe that explanation augmentation attacks perform best on all explainers and datasets. Recall that when the full explanations are available on \grad, the augmentation attacks rank second on the \cora dataset. However, with explanations for the partial node set, augmentation attacks rank first with AUROC and AP of $0.91$ and $0.92$. Moreover, on the \cora and \credit datasets, \faac and \faaw perform competitively with other attacks. Lastly, as observed in the experiments, when the full explanations are available, attacks on the \grad and \szorro explanations are highly successful, while attacks on the \gnnexp are less successful. Nonetheless, we observe AUC and AP $> 0.76$ on both datasets for \gnnexp.

\begin{table*}
\caption{Attack performance with explanations over partial node set for \cora and \credit datasets. The best performing attack(s) on each explanation method is(are) highlighted in bold, and the second best attack(s) is(are) underlined.}
\label{tab:exp-subset}
\centering
\begin{tabular}{clp{1.70cm}<{\centering}p{1.70cm}<{\centering}p{1.70cm}<{\centering}p{1.70cm}<{\centering}}
\toprule
                \multicolumn{1}{l}{$Exp$} & 
                \multirow{1}{*}{\textbf{Attack}}& 
                \multicolumn{2}{c}{\textbf{\cora}} & \multicolumn{2}{c}{\textbf{\credit}} \\ 
                \cmidrule(r){3-4} \cmidrule(r){5-6} 
                         &  & AUC           & AP            & AUC      & AP \\ \midrule

\parbox[t]{2mm}{\multirow{5}{*}{\rotatebox[origin=c]{90}{\grad}}} 
& \faac & $0.683 \pm 0.02$	& $0.700 \pm 0.03$ & 	$0.821 \pm 0.03$ & $0.851 \pm 0.03$ \\
& \faaw & $0.674 \pm 0.02$	& $0.705 \pm 0.02$ & 	$0.806 \pm 0.03$ & $0.839 \pm 0.02$ \\
& \exhoney & $\mathbf{0.916 \pm 0.03}$	& $\mathbf{0.920 \pm 0.03}$ & 	$\mathbf{0.842 \pm 0.04}$ & $\mathbf{0.875 \pm 0.04}$ \\
& \exohoneyex & $0.865 \pm 0.02$	& $0.871 \pm 0.02$ & 	\underline{$0.840 \pm 0.02$} & \underline{$0.870 \pm 0.03$} \\
& \exoattrsim & $\underline{0.881 \pm 0.01}$	& $\underline{0.913 \pm 0.01}$ & 	$0.835 \pm 0.01$ & $0.854 \pm 0.01$ \\ \midrule

\parbox[t]{2mm}{\multirow{5}{*}{\rotatebox[origin=c]{90}{\szorro}}} 
& \faac & $0.863 \pm 0.04$          & $0.896 \pm 0.04$ & 	\underline{$0.825 \pm 0.03$} & \underline{$0.842 \pm 0.02$} \\
& \faaw & $0.720 \pm 0.03$ & $0.757 \pm 0.04$ & 	$0.819 \pm 0.02$ & $0.840 \pm 0.02$ \\
& \exhoney & $\mathbf{0.901 \pm 0.03}$	& $\mathbf{0.905 \pm 0.03}$ & 	$\mathbf{0.835 \pm 0.02}$ & $\mathbf{0.847 \pm 0.03}$ \\
& \exohoneyex & $0.881 \pm 0.04$	& $0.893 \pm 0.01$ & 	$0.762 \pm 0.03$ & $0.801 \pm 0.02$ \\
& \exoattrsim & \underline{$0.887 \pm 0.02$}	& $\underline{0.901 \pm 0.02}$  & 	$0.701 \pm 0.01$ & $0.721 \pm 0.01$ \\ \midrule

\parbox[t]{2mm}{\multirow{5}{*}{\rotatebox[origin=c]{90}{\gnnexp}}} 
& \faac & $0.585 \pm 0.04$          & $0.600 \pm 0.04$ & 	$\mathbf{0.829 \pm 0.03}$ & $\mathbf{0.840 \pm 0.02}$ \\
& \faaw & \underline{$0.680 \pm 0.03$} & \underline{$0.699 \pm 0.03$} & 	\underline{$0.794 \pm 0.02$} & \underline{$0.802 \pm 0.02$} \\
& \exhoney & $\mathbf{0.762 \pm 0.04}$	& $\mathbf{0.791 \pm 0.06}$ & 	$0.782 \pm 0.02$ & $0.799 \pm 0.02$ \\
& \exohoneyex & $0.500 \pm 0.06$	& $0.520 \pm 0.05$& 	$0.688 \pm 0.04$ & $0.715 \pm 0.04$ \\
& \exoattrsim & $0.508 \pm 0.02$	& $0.515 \pm 0.02$ & 	$0.615 \pm 0.02$ & $0.660 \pm 0.02$ \\ \bottomrule

\end{tabular}
\end{table*}

\subsection{When Do Augmentation Attacks Work? }
\label{sec:when-aug-works}
To summarize, augmentation attacks perform best on all explanation models except on \grad and \gradinput across all datasets. Nonetheless, augmentation attack \exhoney on the \grad and \gradinput explanations perform comparable to or better than \exohoneyex attack, which only uses explanations as input on \cseer, \coraml, \bitcoin, and \credit datasets.  
    
On explainers other than \grad and \gradinput, augmentation attacks perform the best. For instance, on the \cseer dataset, augmentation attack \exhoney achieved AUROC and AP of $0.94$ and $0.95$ respectively on \szorro explanation, which is far better than all baselines and explanation-only attacks. The multi-task learning paradigm that \exhoney employs gives it an edge over other augmentation attacks, baselines, and explanation-only attacks. 
    
Overall, the strength of augmentation attacks depends on the choice of the explanation method and the strategy of using the additional feature information.
For instance, \faac and \faaw augmentation attacks, which combine the additional feature information by concatenation and element-wise multiplication, perform worst on the \grad, \gradinput, \zorro, and \szorro explanations. However, they perform best on surrogate method \glime and perturbation-based method \gnnexp. Moreover, augmentation attacks show the best performance when explanations only for a partial node set are available.

\section{Defense}
\mpara{Explanation Perturbation.} To limit the information leakage by the explanation, we perturb each explanation bit using a randomized response mechanism \cite{kairouz2016discrete, wang2016using}. Specifically for 0/1 feature (explanation) mask as in \zorro, we flip each bit of the explanation with a probability that depends on the privacy budget $\epsilon$ as follows
\begin{equation}\label{eq:randomized-response}
	Pr(\mathbf{\mathcal{E}}_{x_i}^\prime =1) = \begin{cases}
		\frac{e^\epsilon}{e^\epsilon + 1}, \quad \text{if}~\mathbf{\mathcal{E}}_{x_i}=1, \\
		\frac{1}{e^\epsilon + 1}, \quad \text{if}~\mathbf{\mathcal{E}}_{x_i}=0,
	\end{cases}
\end{equation}
where $\mathbf{\mathcal{E}}_{x_i}$ and $\mathbf{\mathcal{E}}_{x_i}^\prime$ are true and perturbed $i^{th}$ bit of explanation $\mathbf{\mathcal{E}}_{x}$ respectively. Note that our defense mechanism satisfies $d\epsilon$-local differential privacy. \begin{lemma}
\label{lemma:randres}
For an explanation with $d$ dimensions, the explanation perturbation defense mechanism in Equation \ref{eq:randomized-response} satisfies $d\epsilon$-local differential privacy.
\end{lemma}
\begin{proof}
Note that for an explanation corresponding to two graph datasets $D$ and $D^\prime$ differing in a single edge, the ratio of probabilities of obtaining a certain explanation can be bounded as follows.
$$\frac{Pr[\mathcal{E}_X(D)= S]}{Pr[\mathcal{E}_X(D^\prime) = S]} =\prod_{i=1}^d \frac{Pr[\mathbf{\mathcal{E}}_{x_i}(D)= S_i]}{Pr[\mathbf{\mathcal{E}}_{x_i}(D^\prime) = S_i]} \le \prod_{i=1}^d \frac{\frac{e^{\epsilon}}{e^{\epsilon}+1}}{\frac{1}{e^{\epsilon} +1}} =e^{d\epsilon}.$$
\end{proof}

\mpara{Defense Evaluation.} We evaluate our defense mechanism on \\\exoattrsim attack as it best quantifies the information leakage due to the explanations alone. All other attacks assume the availability of other information, such as features and labels. We use two datasets: \cora and \coraml. As evaluation metrics, we use the AUC score and AP to compute the attack success rate after the defense. Besides, we measure the utility of the perturbed explanation in terms of fidelity, sparsity, and the percentage of 1 bit that is retained from the original explanation (intersection).

\mpara{Defense Results.} 
As shown in Figure \ref{fig:zorro-defense-cora}, the explanation perturbation based on the randomized response mechanism clearly defends against the attack. For instance, at a very high privacy level $\epsilon = 0.0001$, which gives $d\epsilon= 0.14$, the attack performance drastically dropped to 0.56 in AUC and 0.59 in AP, which is about 36\% decrease over the non-private released explanation. As expected, the attack performance decreases significantly with increase in the amount of noise ($\epsilon$ decreases). In Table \ref{tab:defense-fidelity-sparsity-cora}, we analyze the change in explanation utility due to our perturbation mechanism. We observe that on \cora, with the lowest privacy loss level, there is a drop of $5.61\%$ in the fidelity when the attack is already reduced to a random guess. The entropy of the mask distribution increases. In other words, the explanation sparsity decreases. For \zorro, this implies that more bits are set to 1 than in the true explanation mask. Even though this decreases explanation utility to some extent, we point out that $74.68\%$ of true explanation is still retained. Moreover, the sparsity is still lower than achieved by \gnnexp explanations, even without any perturbations.

While quantitatively, the change in explanation sparsity is acceptable, more application-dependent qualitative studies are required to evaluate the change in the utility of explanations. Nevertheless, we provide a promising first defense for future development and possible improvements. 
We obtain similar results for \coraml, which are provided in Appendix \ref{sec:defense-coraml}.

\mpara{Defense Variant for Soft Explanation Masks.} Note that Equation \ref{eq:randomized-response} only applies to explanations that return binary values. For explanations with continuous values, we can adapt the defense as follows. We keep the original value ($\mathbf{\mathcal{E}_x}_i$) when the flipped coin lands heads, but when it lands tail, we replace ($\mathbf{\mathcal{E}_x}_i$) with ${\mathcal{E}_x}^\prime_i$ where ${\mathcal{E}_x}^\prime_i$ is a random number drawn from a normal distribution ($\mathbf{\mathcal{E}_x}^\prime_i \sim \mathcal{N}(0,\,1)$).

\begin{table}[]
\caption{Fidelity, sparsity and percentage of 1 bits in the true explanation that is retained in the perturbed explanation (intersection) after defense for different $\epsilon$ on the \cora dataset for \zorro explanation. $\infty$ implies no privacy.}
\label{tab:defense-fidelity-sparsity-cora}
\begin{tabular}{lccc} \toprule
\textbf{$\epsilon$} & \textbf{Fidelity} & \textbf{Sparsity} & \textbf{Intersection}\\ \midrule
\textbf{0.0001}                   & 0.84        & 5.91     & 74.68                        \\
\textbf{0.001}              & 0.84          & 5.91    & 74.70                      \\
\textbf{0.01}                  & 0.84           & 5.89 & 75.03                             \\
\textbf{0.1}                & 0.84           & 5.80    & 75.10                           \\
\textbf{0.2}              & 0.83         & 5.71    & 75.60                      \\
\textbf{0.4}              & 0.82         & 5.49    & 76.45                         \\
\textbf{0.6}           & 0.81          & 5.25      & 77.16                      
\\ 
\textbf{0.8}           & 0.81          & 5.00    & 78.66                 
\\ 
\textbf{1}           & 0.81          & 4.73    & 80.10                      
\\ 
$\infty$           & 0.89          & 1.83      & 100\\ \bottomrule
\end{tabular}
\end{table}

\begin{figure}
    \centering
    \includegraphics[width=0.47\textwidth]{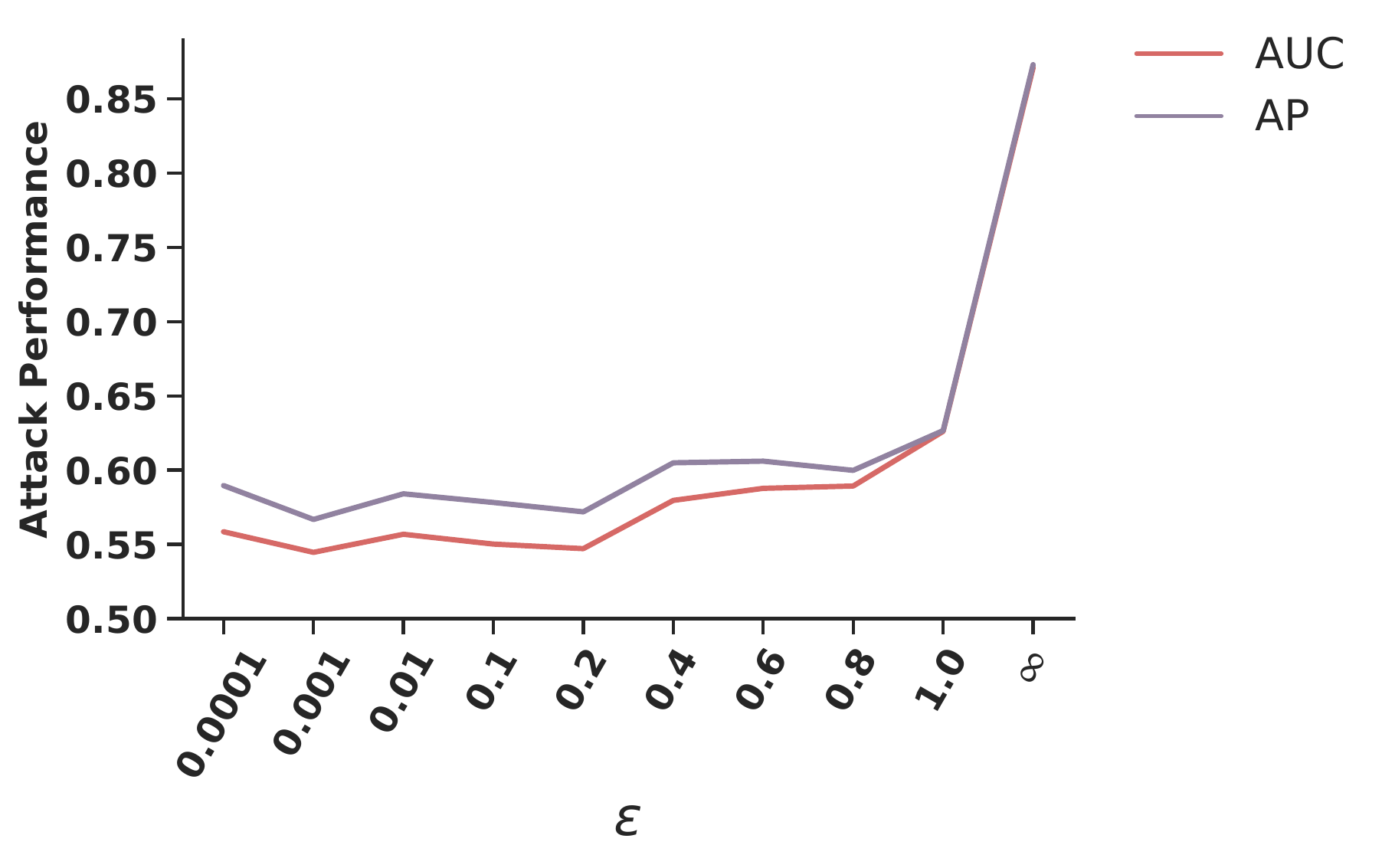}
    \caption{Privacy budget and corresponding attack performance of \exoattrsim for \zorro explanation on the \cora dataset. $\infty$ implies that no perturbation is performed.}
    \label{fig:zorro-defense-cora}
\end{figure}

\section{Conclusion}
We initiate the first investigation on the privacy risks of releasing post-hoc explanations of graph neural networks. Concretely, we quantify the information leakage of explanations via our proposed \emph{five} graph reconstruction attacks. The goal of the attacker is to reconstruct the private graph structure information used to train a GNN model. Our results show that even when the explanations alone are available without any additional auxiliary information, the attacker can reconstruct the graph structure with an AUC score of more than $90\%$. Our explanation-based attacks outperform all baseline methods pointing to the additional privacy risk of releasing explanations. We propose a perturbation-based defense mechanism that reduces the attack to a random guess. The defense leads to a slight decrease in fidelity. At the lowest privacy loss, the perturbed explanation still contains around 75\% of the true explanation. While quantitatively, the change in explanation sparsity seems to be acceptable, more application-dependent qualitative studies would be required to evaluate the change in the utility of explanations. 

We emphasize that we strongly believe in the transparency of graph machine learning and acknowledge the need to explain trained models. At the same time, our work points out the associated privacy risks which cannot be ignored. We believe that our work would encourage future work on finding solutions to balance the complex trade-off between privacy and transparency.

\begin{acks}
This work is partly funded by the Lower Saxony Ministry of Science and Culture under grant number ZN3491 within the Lower Saxony "Vorab" of the Volkswagen Foundation and supported by the Center for Digital Innovations (ZDIN), and the Federal Ministry of Education and Research
(BMBF), Germany under the project LeibnizKILabor (grant number
01DD20003). The authors are grateful to the anonymous reviewers for providing insights that further improved our paper.
\end{acks}

\bibliographystyle{ACM-Reference-Format}
\bibliography{main}

\appendix

\section*{Appendix}
\mpara{Organization.} The Appendix is organized as follows. We first present related works in private graph extraction attacks and other attacks on GNN models. Next, we elaborated on the data sampling used for evaluation in Section \ref{sec:data-sampling-elaborate} followed by the results for defense on the \coraml dataset in Section \ref{sec:defense-coraml}. The performance and the corresponding  standard deviation of our main results are provided in Section \ref{sec:main-result-std}. In Section \ref{sec:additional-experiments}, we present the results of the additional experiments on the \cseer, \pubmed, and \credit datasets. 
Finally, we list the hyperparameters used for our attacks and target model in Section \ref{sec:hyperparam}.

\section{Related Works}
\paragraph{\textbf{Private Graph Extraction Attacks}}
Given a black-box access to a GNN model that is trained on a target dataset and the adversary’s background knowledge, \citet{he2021stealing} proposed link stealing attacks to infer whether there is a link between a given pair of nodes in the target dataset. Their attacks are specific to the adversary's background knowledge which range from simply exploiting the node feature similarities to a shadow model-based attack. \citet{wu2021linkteller} proposed an edge re-identification attack for vertically partitioned graph learning. Their attack setting is different from ours and is applicable in scenarios where the high-dimensional features and high-order adjacency information are usually heterogeneous and held by different data holders. GraphMI \cite{zhanggraphmi} aims at reconstructing the adjacency matrix of the target graph given white-box access to the trained model, node features, and labels. 

\paragraph{\textbf{Other Inference Attacks and Defenses on GNNs}}
Several other attacks, such as membership inference \cite{olatunji2021membership, duddu2020quantifying} and model extraction attacks \cite{wu2020model} have been proposed to quantify privacy leakage in GNNs.
In a membership inference attack, the goal of the attacker is to infer whether a node was part of the data used in training the GNN model via black-box access to the trained model. In a model extraction attack, the attacker aims to steal the trained model's parameter and hyperparameters to duplicate or mimic the functionality of the target model via the predictions returned from querying the model\cite{wu2020model}. Recently, several defenses against these attacks have been proposed, which are mainly based on differential privacy. \citet{olatunji2021releasing} proposed a method for releasing GNN models by combining a knowledge-distillation framework with two noise mechanisms, random subsampling and noisy labeling. Their centralized setting approach trains a student model using a public graph and private labels obtained from a teacher model trained exclusively for each query node (personalized teacher models). Their method, by design, defends against membership inference attacks and model extraction attacks since only the student model (which has limited and perturbed information) is released. \citet{sajadmanesh2020locally} proposed a locally differentially private GNN model by considering a distributed setting where nodes and labels are private, and the graph structure is known to the central server. Their approach perturbs both the node features and labels to ensure a differential privacy guarantee. However, all attacks and defenses are not applicable to explanations.

\paragraph{\textbf{Membership Inference Attack and Explanations}}
On eucli\-dean data such as images, \citet{shokri2021privacy} analyzed the privacy risks of feature-based model explanations using membership inference attacks which  quantifies the extent to which model predictions and their explanations leak information about the presence of a datapoint in the training set of a model. We emphasize that the goal of \citet{shokri2021privacy} differs from ours in that we focus on reconstructing the entire graph structure from feature-based explanations. Also, their investigations are limited to non-graph data and the corresponding target and explanation models.

\paragraph{\textbf{Adversarial Attacks and GNNs.}}
Another line of research focuses on the vulnerability of GNNs to adversarial attacks \cite{wu2019adversarial, zhang2020backdoor, zugner2018adversarial, zugner2019adversarial, dai2018adversarial, wang2019attacking}. The goal of the attacker is to fool the GNN model into making a wrong prediction by manipulating node features or the structural information of nodes. 
A recent work \cite{fan2021jointly} used explanation method such as GNNExplainer as a method for detecting adversarial perturbation on graphs. Hence, acting as a tool for inspecting adversarial attacks on GNN models. They further proposed an adversarial attack framework (GEAttack) that exploits the vulnerabilities of explanation methods and the GNN model. This allows the attacker to simultaneously fool the GNN model and misguide the inspection from the explanation method. Our work differs significantly from this work in that first, we aim to reconstruct the graph from the explanations and, secondly, to quantify the privacy leakage of explanations on GNN models.

\section{Data sampling for evaluation}
\label{sec:data-sampling-elaborate}
In this section, we elaborate on the data sampling method used for our evaluation as briefly explained in Section \ref{sec:eval_metrics}. 
To generate a test set, we randomly select $10\%$ of all nodes. We use the set of all edges which are incident on at least one of the nodes in our selected set. We then sample an equal number of negative edges (pairs of nodes such that no edge exists between them; one of the nodes in this pair is from the selected set of $10\%$ nodes).
In total, we generate 10 random test sets.

We employ 10 random instantiations of the attack models.
Corresponding to each instantiation of the attack model, we test the model with one of our random test sets. In total, we have $10$ observations of the final scores. We evaluated all methods on the same test sets. As we wanted to keep the number of observations the same for all compared methods, including the non-parameterized methods based on feature and explanation similarity (note that they do not use a neural model, and there is no initialization required in their cases), we tested each attack model (corresponding to each initialization) on one test set. We computed the average precision (AP) and AUROC result on the balanced pairs. For the partial explanation experiment, we randomly sample 30\% of the nodes only once to generate a fixed subgraph. Evaluation is performed as in the previous case. Here, the $10$ test sets are generated from this subgraph.

We remark that using AUC and AP scores allows us to evaluate all methods across all ranges of cutoff decision thresholds. Our evaluation is domain and dataset-independent. Nevertheless, in real scenarios, an attacker might need to choose a single decision threshold for which domain-specific knowledge or an auxiliary validation dataset might be required.

\section{Results for defense on \coraml}
\label{sec:defense-coraml}
The attack performance for different values of $\epsilon$ is plotted in Figure  \ref{fig:zorro-defense-coraml}. For the lowest privacy budget, we observe that the attack is reduced to a random guess (with an AUC score close to 0.55). The variation in explanation utility and intersection with true explanation is shown in Table \ref{tab:defense-fidelity-sparsity-coraml}. Here also, the perturbed explanation is able to retain around 75\% of the 1 bit of the true explanation. While there is a drop in fidelity and the explanation becomes denser, we note that the perturbed explanation still shows higher fidelity and sparsity than other explanation methods (c.f. Table \ref{tab:rdt-fidelity-sparsity}).

\begin{figure}[h!!]
    \centering
    \includegraphics[width=0.47\textwidth]{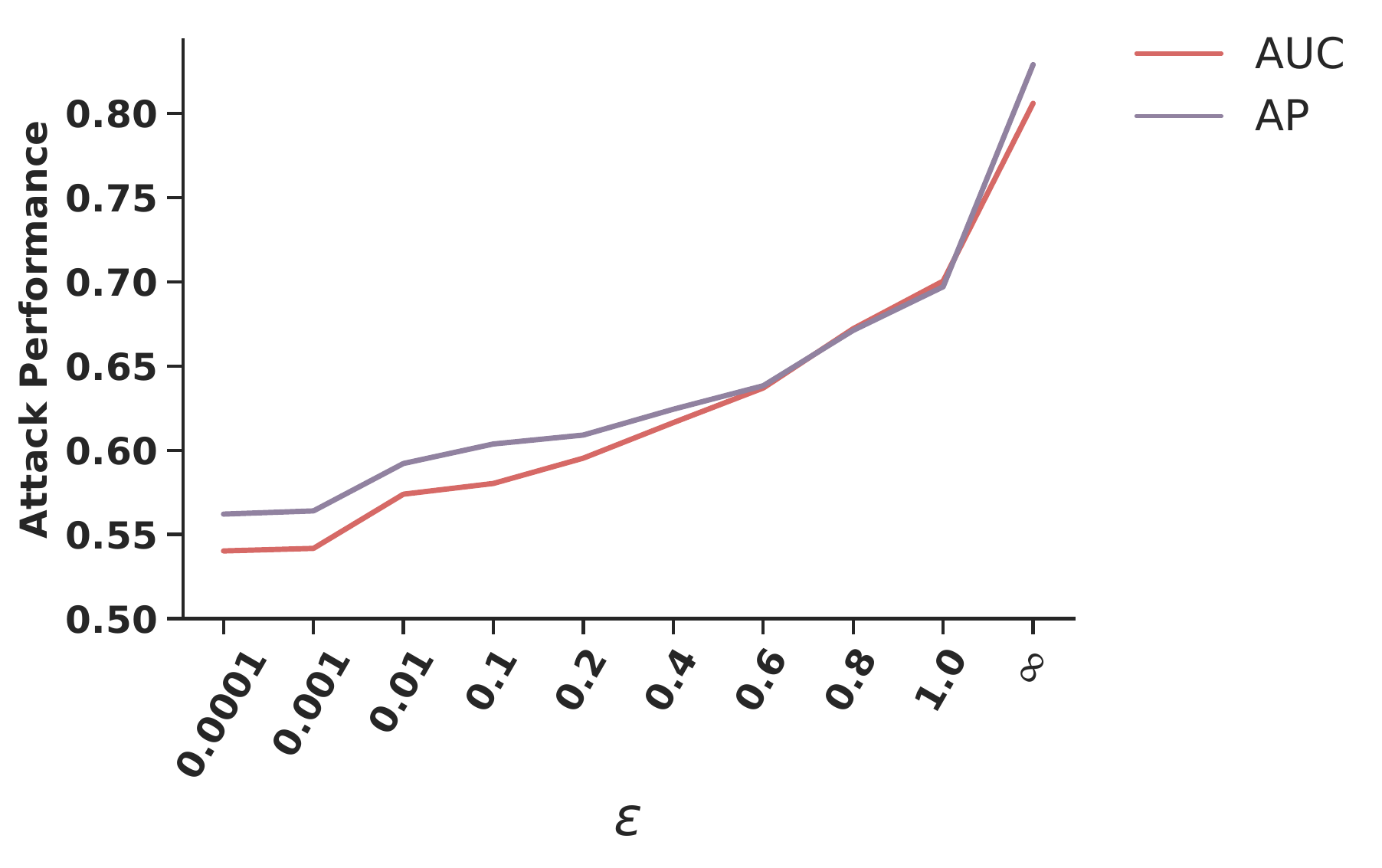}
    \caption{Privacy budget and corresponding attack performance of \exoattrsim for \zorro explanation on the \coraml dataset. $\infty$ implies no privacy.}
    \label{fig:zorro-defense-coraml}
\end{figure}

\begin{table}[h!]
\caption{Fidelity, sparsity and percentage of 1 bits in the true explanation that is retained in the perturbed explanation (intersection) after defense for different $\epsilon$ on the \coraml dataset for \zorro explanation. $\infty$ implies no privacy.}
\label{tab:defense-fidelity-sparsity-coraml}
\begin{tabular}{lccc} \toprule
\textbf{$\epsilon$} & \textbf{Fidelity} & \textbf{Sparsity} & \textbf{Intersection}\\ \midrule
\textbf{0.0001}                   & 0.86        & 4.53     & 74.96                        \\
\textbf{0.001}              & 0.86          & 4.53    & 74.98                      \\
\textbf{0.01}                  & 0.86           & 4.53 & 75.12                             \\
\textbf{0.1}                & 0.87           & 4.46    & 75.08                           \\
\textbf{0.2}              & 0.87         & 4.39    & 75.20                      \\
\textbf{0.4}              & 0.87         & 4.24    & 75.70                         \\
\textbf{0.6}           & 0.87          & 4.08      & 77.07                      
\\ 
\textbf{0.8}           & 0.88          & 3.95    & 78.70                 
\\ 
\textbf{1}           & 0.89          & 3.81    & 80.75                      
\\ 
$\infty$           & 0.96          & 3.33      & 100 \\ \bottomrule
\end{tabular}
\end{table}

\section{Detailed Results }
\label{sec:main-result-std}
In Table \ref{tab:main-result-std}, we present the mean and standard deviation of the results in Table \ref{tab:main-result}. Since our results have a low standard deviation across the 10 runs of the experiment, our method is stable.

\begin{table*}[h!!]
\caption{Results with standard deviation of Table \ref{tab:main-result}. The best performing attack(s) on each explanation method is(are) highlighted in bold, and the second best attack(s) is(are) underlined.}
\label{tab:main-result-std}
\centering
\begin{tabular}{clp{1.4cm}<{\centering}p{1.4cm}<{\centering}p{1.4cm}<{\centering}p{1.4cm}<{\centering}p{1.4cm}<{\centering}p{1.4cm}<{\centering}}
\toprule
                \multicolumn{1}{l}{$Exp$} & 
                \multirow{1}{*}{\textbf{Attack}}& 
                \multicolumn{2}{c}{\textbf{\cora}} & \multicolumn{2}{c}{\textbf{\coraml}} & \multicolumn{2}{c}{\textbf{\bitcoin}}\\ 
                \cmidrule(r){3-4} \cmidrule(r){5-6} \cmidrule(r){7-8} 
                         &  & AUC           & AP            & AUC      & AP & AUC      & AP \\ \midrule

\parbox[t]{2mm}{\multirow{4}{*}{\rotatebox[origin=c]{90}{Baseline}}}
& \attrsim & $0.799 \pm 0.04$	& $0.827 \pm 0.04$ & $0.706 \pm 0.08$	& $0.753 \pm 0.07$ & $0.535 \pm 0.03$	& $0.478 \pm 0.02$\\
& \lsa \cite{he2021stealing} & $ 0.795 \pm 0.03$	& $ 0.810 \pm 0.02$ & $ 0.725 \pm 0.04$	& $0.760 \pm 0.01$ & $0.532 \pm 0.05$	& $0.500 \pm 0.06$\\
& \graphmi \cite{zhanggraphmi} & $0.856 \pm 0.01$	& $0.830 \pm 0.01$ & $0.808 \pm 0.02$	& $0.814 \pm 0.03$ & $0.585 \pm 0.05$	& $ 0.518 \pm 0.06$\\ 
& \gsl \cite{fatemi2021slaps} & $0.736 \pm 0.05$	& $0.776 \pm 0.05$ & $0.649 \pm 0.06$	& $0.702 \pm 0.07$ & $0.597 \pm 0.09$	& $0.577 \pm 0.07$\\ \midrule

\parbox[t]{2mm}{\multirow{5}{*}{\rotatebox[origin=c]{90}{\grad}}} 
& \faac & $0.734 \pm 0.05$	& $0.773 \pm 0.04$      & $0.640 \pm 0.05$	& $0.705 \pm 0.04$ & $0.527 \pm 0.04$	& $0.515 \pm 0.03$\\
& \faaw &$0.678 \pm 0.04$	& $0.737 \pm 0.03$ & $0.666 \pm 0.06$	& $0.730 \pm 0.05$ & $0.264 \pm 0.07$	& $0.383 \pm 0.04$\\
& \exhoney &\underline{$0.948 \pm 0.02$}	& \underline{$0.953 \pm 0.01$} & $\mathbf{0.902 \pm 0.08}$	& \underline{$0.833 \pm 0.07$} & $\mathbf{0.700 \pm 0.05}$	& $\mathbf{0.715 \pm 0.04}$\\
& \exohoneyex &$0.924 \pm 0.03$	& $0.939 \pm 0.02$ & $0.699 \pm 0.07$	& $0.768 \pm 0.05$ & $0.229 \pm 0.03$	& $0.365 \pm 0.02$\\
& \exoattrsim & $\mathbf{0.984 \pm 0.01}$	& $\mathbf{0.978 \pm 0.01}$ & \underline{$0.890 \pm 0.04$} & $\mathbf{0.891 \pm 0.04}$ & \underline{$0.681 \pm 0.03$}	& \underline{$0.644 \pm 0.03$}\\ \midrule

\parbox[t]{2mm}{\multirow{5}{*}{\rotatebox[origin=c]{90}{\gradinput}}} 
& \faac & $0.734 \pm 0.06$          & $0.775 \pm 0.04$          & $0.674 \pm 0.05$	& $0.724 \pm 0.04$ & $0.525 \pm 0.09$	& $0.527 \pm 0.05$\\ 
& \faaw & $0.691 \pm 0.02$ & $0.742 \pm 0.02$ & $0.717 \pm 0.05$ & $0.756 \pm 0.06$ & $0.252 \pm 0.03$	& $0.380 \pm 0.02$\\
& \exhoney & \underline{$0.949 \pm 0.02$}	& \underline{$0.950 \pm 0.02$} & \underline{$0.787 \pm 0.08$}	& \underline{$0.832 \pm 0.07$} & $\mathbf{0.709 \pm 0.04}$	& $\mathbf{0.723 \pm 0.03}$\\
& \exohoneyex &$0.903 \pm 0.04$	& $0.923 \pm 0.04$ &$0.717 \pm 0.08$	& $0.781 \pm 0.06$ & $0.256 \pm 0.03$	& $0.380 \pm 0.02$\\
& \exoattrsim &$\mathbf{0.984 \pm 0.01}$	& $\mathbf{0.979 \pm 0.01}$ & $\mathbf{0.903 \pm 0.04}$	& $\mathbf{0.899 \pm 0.04}$ & \underline{$0.681 \pm 0.03$}	& \underline{$0.644 \pm 0.03$}\\ \midrule

\parbox[t]{2mm}{\multirow{5}{*}{\rotatebox[origin=c]{90}{\zorro}}} 
& \faac & $0.823 \pm 0.04$	& $0.860 \pm 0.05$          & $0.735 \pm 0.02$	& $0.786 \pm 0.01$ & \underline{$0.575 \pm 0.03$} & $0.529 \pm 0.05$\\ 
& \faaw & $0.723 \pm 0.03$	& $0.756 \pm 0.03$ & $0.681 \pm 0.02$	& $0.697 \pm 0.04$ & $0.399 \pm 0.07$ & $0.449 \pm 0.05$\\
& \exhoney & $\mathbf{0.884 \pm 0.03}$	& $\mathbf{0.880 \pm 0.04}$ & \underline{$0.776 \pm 0.03$}	& \underline{$0.820 \pm 0.02$} & $0.537 \pm 0.05$ & \underline{$0.527 \pm 0.04$}\\
& \exohoneyex & $0.779 \pm 0.04$	& $0.810 \pm 0.01$ & $0.722 \pm 0.02$	& $0.777 \pm 0.02$ & $\mathbf{0.596 \pm 0.03}$ & $\mathbf{0.561 \pm 0.03}$\\
& \exoattrsim & \underline{$0.871 \pm 0.02$}	& \underline{$0.873 \pm 0.02$} & $\mathbf{0.806 \pm 0.02}$ & $\mathbf{0.829 \pm 0.03}$ & $0.427 \pm 0.06$ & $0.485 \pm 0.05$\\ \midrule

\parbox[t]{2mm}{\multirow{5}{*}{\rotatebox[origin=c]{90}{\szorro}}} 
& \faac & $0.907 \pm 0.03$          & $0.922 \pm 0.02$          & \underline{$0.747 \pm 0.06$}	& \underline{$0.791 \pm 0.05$} & $\mathbf{0.601 \pm 0.06}$ & $\mathbf{0.590 \pm 0.05}$\\ 
& \faaw & $0.794 \pm 0.06$ & $0.815 \pm 0.06$ & $0.712 \pm 0.06$	& $0.740 \pm 0.06$ & $0.490 \pm 0.08$ & $0.491 \pm 0.05$\\
& \exhoney & $\mathbf{0.918 \pm 0.02}$	& \underline{$0.923 \pm 0.02$} & $\mathbf{0.776 \pm 0.06}$	& $0.819 \pm 0.05$ & \underline{$0.598 \pm 0.03$} & \underline{$0.565 \pm 0.03$}\\
& \exohoneyex & $0.893 \pm 0.04$	& $0.915 \pm 0.02$ & $0.742 \pm 0.06$	& $\mathbf{0.784 \pm 0.05}$ & $0.571 \pm 0.03$ & $0.564 \pm 0.04$\\
& \exoattrsim & \underline{$0.908 \pm 0.03$}	& $\mathbf{0.934 \pm 0.02}$ & $0.732 \pm 0.05$	& $0.787 \pm 0.03$ & $0.484 \pm 0.04$ & $0.496 \pm 0.03$\\ \midrule

\parbox[t]{2mm}{\multirow{5}{*}{\rotatebox[origin=c]{90}{\glime}}} 
& \faac & \underline{$0.643 \pm 0.05$}          & \underline{$0.710 \pm 0.04$}          & \underline{$0.610 \pm 0.05$}	& \underline{$0.652 \pm 0.04$} & \underline{$0.473 \pm 0.07$} & \underline{$0.492 \pm 0.05$}\\ 
& \faaw & $0.516 \pm 0.06$ & $0.522 \pm 0.04$ & $0.517 \pm 0.05$	& $0.528 \pm 0.04$ & $0.264 \pm 0.03$ & $0.371 \pm 0.01$\\
& \exhoney & $\mathbf{0.730 \pm 0.05}$	& $\mathbf{0.774 \pm 0.03}$ & $\mathbf{0.681 \pm 0.05}$	& $\mathbf{0.740 \pm 0.05}$ & $\mathbf{0.542 \pm 0.05}$ & $\mathbf{0.525 \pm 0.03}$\\
& \exohoneyex &$0.558 \pm 0.06$	& $0.571 \pm 0.05$ & $0.540 \pm 0.06$	& $0.555 \pm 0.05$ & $0.236 \pm 0.04$ & $0.361 \pm 0.02$\\
& \exoattrsim &$0.505 \pm 0.04$	& $0.524 \pm 0.04$ & $0.520 \pm 0.04$	& $0.523 \pm 0.04$ & $0.504 \pm 0.05$ & $0.512 \pm 0.03$\\ \midrule

\parbox[t]{2mm}{\multirow{5}{*}{\rotatebox[origin=c]{90}{\gnnexp}}} 
& \faac & $0.614 \pm 0.04$          & $0.650 \pm 0.04$          & $0.653 \pm 0.05$	& $0.705 \pm 0.04$ & $0.467 \pm 0.11$ & $0.489 \pm 0.06$\\ 
& \faaw & \underline{$0.724 \pm 0.05$} & \underline{$0.760 \pm 0.05$} & \underline{$0.637 \pm 0.05$}	& \underline{$0.692 \pm 0.05$} & $0.390 \pm 0.10$ & $0.454 \pm 0.05$\\
& \exhoney & $\mathbf{0.762 \pm 0.04}$	& $\mathbf{0.796 \pm 0.03}$ & $\mathbf{0.700 \pm 0.07}$	& $\mathbf{0.796 \pm 0.06}$ & $\mathbf{0.590 \pm 0.04}$ & $\mathbf{0.563 \pm 0.03}$\\
& \exohoneyex &$0.517 \pm 0.04$	& $0.552 \pm 0.03$ & $0.490 \pm 0.05$	& $0.508 \pm 0.04$ & $0.386 \pm 0.11$ & $0.451 \pm 0.07$\\
& \exoattrsim &$0.537 \pm 0.05$	& $0.541 \pm 0.04$ & $0.484 \pm 0.05$	& $0.508 \pm 0.04$ & \underline{$0.551 \pm 0.04$} & \underline{$0.545 \pm 0.03$}\\ \bottomrule

\end{tabular}
\end{table*}

\section{Additional Experiments}
\label{sec:additional-experiments}
We perform additional experiments on 3 additional datasets to verify our observations and the effectiveness of our attacks. We start by describing the additional datasets. The performance scores are provided in Table \ref{tab:additional-exp}, and results are discussed in the main paper.
\begin{table}[h!!]
\caption{Dataset statistics for the additional datasets. $|V|$ and $|E|$ denotes the number of nodes and edges respectively, $\mathcal{C}$, $\mathbf{X}_d$, and \textbf{deg} denotes the number of classes, size of feature dimension and the average degree of the corresponding graph dataset.}
\label{tab:data-stat-additional}
\begin{tabular}{lccc}
\toprule
& \textbf{\cseer} & \textbf{\pubmed} & \textbf{\credit} \\ \midrule
$|V|$    & 3327     & 19717        & 3000 \\
$|E|$   & 9228      & 88651    & 28854 \\
$\mathbf{X}_d$ & 3703   & 500    & 13 \\
$\mathcal{C}$  & 6                 & 3             & 2 \\
\textbf{deg}  & 2.77                 & 4.50             & 9.62 \\\bottomrule
\end{tabular}
\end{table}
\subsection{Additional Datasets}
\label{sec:additional-datasets}
We used \cseer and \pubmed, which are conventional datasets for node classification tasks and the \credit dataset.
The details about the datasets are in Table \ref{tab:data-stat-additional}. The fidelity and sparsity for the corresponding explanations are provided  in Table \ref{tab:rdt-fidelity-sparsity-additional-experiment}.

\mpara{\pubmed and \cseer.} Both \pubmed and \cseer~\cite{sen2008collective} are citation datasets, where each node is a research article, and there is an edge between two articles if one cites the other. In \cseer, features are binary values like \cora, while in \pubmed, the features are represented by the TF/IDF weighted word vector of the unique words in the dictionary.

\mpara{\credit.} The Credit defaulter graph \cite{agarwal2021towards} is a financial network where an edge exists between individuals (nodes) if there is a similarity in their spending and payment pattern. The task is to determine whether an individual will default on their credit card payment. We used a subset of this credit dataset which has 3000 nodes.

\begin{table}[h!!]
\caption{RDT-Fidelity and sparsity (entropy) of different explanation methods on the additional datasets. For fidelity, the higher the better. For sparsity, the lower the better.}
\label{tab:rdt-fidelity-sparsity-additional-experiment}
\begin{tabular}{lp{0.7cm}<{\centering}p{0.7cm}<{\centering}p{0.7cm}<{\centering}p{0.7cm}<{\centering}p{0.7cm}<{\centering}p{0.7cm}<{\centering}} 
\toprule
$Exp$    & \multicolumn{2}{c}{\cseer} & \multicolumn{2}{c}{\pubmed} & \multicolumn{2}{c}{\credit}  \\ 
\cmidrule(r){2-3} \cmidrule(r){4-5} \cmidrule(r){6-7}
       & \small{Fidelity}    & \small{Sparsity}   & \small{Fidelity}     & \small{Sparsity}    & \small{Fidelity}     & \small{Sparsity}    \\ \midrule
\textbf{\grad}   & 0.24 & 4.15 & 0.44 & 4.38 & 0.55 & 0.94  \\
\textbf{\gradinput} & 0.19 & 4.16 & 0.37 & 4.39  & 0.59 & 0.95\\
\textbf{\szorro} &0.90 &2.62 &0.96 &2.17  &0.94  &1.23  \\
\textbf{\glime} &0.18 &0.65 &0.38 &1.59  &0.60  &0.49 \\
\textbf{\gnnexp} & 0.65 & 8.21 & 0.79 & 6.21 & 0.65 & 2.56 \\ \bottomrule

\end{tabular}
\end{table}

\begin{table*}[h!!]
\caption{Attack performance and baselines for additional datasets (\cseer, \pubmed, and \credit). The best performing attack(s) on each explanation method is(are) highlighted in bold, and the second best attack(s) is(are) underlined.}
\label{tab:additional-exp}
\centering
\begin{tabular}{clp{1.55cm}<{\centering}p{1.55cm}<{\centering}p{1.55cm}<{\centering}p{1.55cm}<{\centering}p{1.55cm}<{\centering}p{1.55cm}<{\centering}}
\toprule
                \multicolumn{1}{l}{$Exp$} & 
                \multirow{1}{*}{\textbf{Attack}}& 
                \multicolumn{2}{c}{\textbf{\cseer}} & \multicolumn{2}{c}{\textbf{\pubmed}} & \multicolumn{2}{c}{\textbf{\credit}} \\ 
                \cmidrule(r){3-4} \cmidrule(r){5-6} \cmidrule(r){7-8}
                         &  & AUC           & AP            & AUC      & AP & AUC      & AP \\ \midrule

\parbox[t]{2mm}{\multirow{4}{*}{\rotatebox[origin=c]{90}{Baseline}}}
& \attrsim & $0.890 \pm 0.04$	&  $0.909 \pm 0.03$& $0.888 \pm 0.01$	& $0.889 \pm 0.01$ & 	$0.720 \pm 0.01$ & $0.753 \pm 0.01$ \\
& \lsa \cite{he2021stealing} & $0.859 \pm 0.03$	&  $0.893 \pm 0.02$& $0.783 \pm 0.05$	& $0.801 \pm 0.04$ & 	$0.655 \pm 0.03$ & $0.671 \pm 0.04$ \\
& \graphmi \cite{zhanggraphmi} & $0.789 \pm 0.02$	&  $0.746 \pm 0.02$& $0.795 \pm 0.01$	& $0.758 \pm 0.02$ & 	$0.510 \pm 0.04$ & $0.569 \pm 0.05$ \\ 
& \gsl \cite{fatemi2021slaps} & $0.834 \pm 0.04$	&  $0.875 \pm 0.03$& $0.579 \pm 0.02$	& $0.626 \pm 0.02$ & 	$0.723 \pm 0.02$ & $0.767 \pm 0.02$ \\ \midrule

\parbox[t]{2mm}{\multirow{5}{*}{\rotatebox[origin=c]{90}{\grad}}} 
& \faac & $ 0.811 \pm 0.03 $	&  $ 0.860 \pm 0.03 $& $ 0.672 \pm 0.02$	& $0.732 \pm 0.01$ & 	$ 0.825 \pm 0.02$ & $ 0.868 \pm 0.02$ \\
& \faaw & $0.748 \pm 0.05$	&  $0.808 \pm 0.04$& $ 0.579 \pm 0.03$	& $ 0.627 \pm 0.03$ & 	$0.820 \pm 0.02$ & $0.861 \pm 0.01$ \\
& \exhoney & \underline{$ 0.969 \pm 0.02$}	&  \underline{$0.971 \pm 0.02$} & \underline{$0.788 \pm 0.02$}	& \underline{$0.856 \pm 0.02$} & 	$\mathbf{0.881 \pm 0.02}$ & $\mathbf{0.918 \pm 0.02}$ \\
& \exohoneyex & $ 0.945 \pm 0.03$	&  $0.955 \pm 0.02$& $0.770 \pm 0.03$	& $0.846 \pm 0.02$ & 	$ 0.861 \pm 0.02$ & $0.899 \pm $ 0.01\\
& \exoattrsim & $ \mathbf{0.991 \pm 0.01}$	&  $\mathbf{0.985 \pm 0.01}$ & $\mathbf{0.987 \pm 0.01}$	& $\mathbf{0.981 \pm 0.01}$ & 	\underline{$0.861 \pm 0.01$} & \underline{$0.867 \pm 0.01$} \\ \midrule

\parbox[t]{2mm}{\multirow{5}{*}{\rotatebox[origin=c]{90}{\gradinput}}} 
& \faac & $0.810 \pm 0.04$	&  $0.857 \pm 0.04$& $0.583 \pm 0.02$	& $0.632 \pm 0.02$ & 	$0.826 \pm 0.02$ & $0.867 \pm 0.02$ \\
& \faaw & $0.745 \pm 0.05$	&  $0.804 \pm 0.04$& $0.574 \pm 0.01$	& $0.625 \pm 0.01$ & 	$0.813 \pm 0.02$ & $0.863 \pm 0.01$ \\
& \exhoney & \underline{$0.970 \pm 0.02$}	&  \underline{$0.974 \pm 0.02$}& \underline{$0.782 \pm 0.02$}	& \underline{$0.853 \pm 0.02$} & 	$\mathbf{0.881 \pm 0.02}$ & $\mathbf{0.918 \pm 0.01}$ \\
& \exohoneyex & $0.947 \pm 0.02$	&  $0.960 \pm 0.02$& $0.740 \pm 0.02$	& $0.836 \pm 0.01$ & 	$0.861 \pm 0.02$ & $0.893 \pm 0.02$ \\
& \exoattrsim & $\mathbf{0.992 \pm 0.001}$	&  $\mathbf{0.986 \pm 0.001}$& $\mathbf{0.988 \pm 0.001}$	& $\mathbf{0.981 \pm 0.001}$ & 	\underline{$0.864 \pm 0.001$} & \underline{$0.868 \pm 0.001$}\\ \midrule

\parbox[t]{2mm}{\multirow{5}{*}{\rotatebox[origin=c]{90}{\szorro}}} 
& \faac & $0.932 \pm 0.02$	&  $0.944 \pm 0.02$& $0.663 \pm 0.02$	& $0.729 \pm 0.01$ & 	\underline{$0.825 \pm 0.03$} & \underline{$0.868 \pm 0.02$} \\
& \faaw & $0.828 \pm 0.06$	&  $0.855 \pm 0.06$& $0.554 \pm 0.03$	& $0.601 \pm 0.02$ & 	$0.828 \pm 0.03$ & $0.868 \pm 0.02$ \\
& \exhoney & $\mathbf{0.953 \pm 0.02}$	&  $\mathbf{0.960 \pm 0.02}$& $\mathbf{0.901 \pm 0.02}$	& $\mathbf{0.914 \pm 0.02}$ & 	$\mathbf{0.822 \pm 0.02}$ & $\mathbf{0.869 \pm 0.02}$ \\
& \exohoneyex & $0.932 \pm 0.03$	&  $0.947 \pm 0.03$& $0.753 \pm 0.02$	& $0.772 \pm 0.02$ & 	$0.785 \pm 0.01$ & $0.840 \pm 0.01$ \\
& \exoattrsim & \underline{$0.950 \pm 0.01$}	&  \underline{$0.957 \pm 0.01$} & \underline{$0.874 \pm 0.02$}	& \underline{$0.904 \pm 0.01$} & 	$0.734 \pm 0.02$ & $0.768 \pm 0.02$ \\ \midrule

\parbox[t]{2mm}{\multirow{5}{*}{\rotatebox[origin=c]{90}{\glime}}} 
& \faac & \underline{$0.757 \pm 0.05$}	&  \underline{$0.801 \pm 0.05$} & $0.614 \pm 0.03$	& $0.641 \pm 0.02$ & 	\underline{$0.825 \pm 0.03$} & \underline{$0.867 \pm 0.02$} \\
& \faaw & $0.567 \pm 0.03$	&  $0.637 \pm 0.04$& $0.507 \pm 0.10$	& $0.509 \pm 0.08$ & 	$0.791 \pm 0.01$ & $0.855 \pm 0.01$ \\
& \exhoney & $\mathbf{0.862 \pm 0.05}$	&  $\mathbf{0.894 \pm 0.04}$ & $\mathbf{0.697 \pm 0.02}$	& $\mathbf{0.712 \pm 0.02}$ & 	$\mathbf{0.845 \pm 0.02}$ & $\mathbf{0.890 \pm 0.01}$ \\
& \exohoneyex & $0.571 \pm 0.04$	&  $0.618 \pm 0.04$& $0.602 \pm 0.03$	& $0.591 \pm 0.03$ & 	$0.798 \pm 0.03$ & $0.843 \pm 0.03$ \\
& \exoattrsim & $0.703 \pm 0.06$	&  $0.705 \pm 0.05$& \underline{$0.661 \pm 0.03$}	& \underline{$0.724 \pm 0.02$} & 	$0.765 \pm 0.03$ & $0.810 \pm 0.02$ \\ \midrule

\parbox[t]{2mm}{\multirow{5}{*}{\rotatebox[origin=c]{90}{\gnnexp}}} 
& \faac & $0.639 \pm 0.04$	&  $0.713 \pm 0.04$& $0.552 \pm 0.03$	& $0.590 \pm 0.03$ & 	$\mathbf{0.826 \pm 0.03}$ & $\mathbf{0.869 \pm 0.02}$ \\
& \faaw & \underline{$0.812 \pm 0.05$}	&  \underline{$0.862 \pm 0.04$} & \underline{$0.570 \pm 0.01$}	& \underline{$0.618 \pm 0.01$} & 	\underline{$0.821 \pm 0.02$} & \underline{$0.865 \pm 0.02$} \\
& \exhoney & $\mathbf{0.847 \pm 0.05}$	&  $\mathbf{0.881 \pm 0.05}$& $\mathbf{0.612 \pm 0.02}$	& $\mathbf{0.648 \pm 0.02}$ & 	$0.817 \pm 0.02$ & $0.869 \pm 0.02$ \\
& \exohoneyex & $0.594 \pm 0.07$	&  $0.640 \pm 0.06$& $0.507 \pm 0.02$	& $0.512 \pm 0.02$ & 	$0.760 \pm 0.04$ & $0.804 \pm 0.05$ \\
& \exoattrsim & $0.618 \pm 0.05$	&  $0.654 \pm 0.05$& $0.495 \pm 0.01$	& $0.499 \pm 0.01$ & 	$0.787 \pm 0.02$ & $0.845 \pm 0.02$ \\ \bottomrule

\end{tabular}
\end{table*}

\subsection{Ensuring the Stability of Results}
To ensure the stability of the results, we perform an additional experiment by slightly modifying the evaluation procedure in Section \ref{sec:data-sampling-elaborate} (hereafter referred to as \textit{10-runs}). Instead of testing the attacks on one of our random test sets for each instantiation which resulted in $10$ observations of the final scores, we modify the evaluation as follows. For each instantiation, we evaluate the attacks on \textit{all} $10$ held out balanced test sets. This results in $100$ observations of the final score (recall that we have $10$ instantiations). We report the mean and the standard deviation. This modification which we call \textit{100-runs}, further mitigates any sampling biases from the test sets and any randomness. Similar to 10-runs, we evaluated all methods on the same test sets. The \exoattrsim attack is excluded since there is no randomness involved, resulting in the same results as in 10-runs.

We observe that in most cases, the difference between all previous experiments, which were evaluated using the 10-runs procedure and the modified 100-runs is $0$. This shows that our experiments are stable and the results are reliable. Although we observe some minor variations in some results, however, such variations are relatively small and do not invalidate our observations or experiments. For example, on the \pubmed dataset, \exohoneyex attack on the \grad explanation for the 10-runs resulted in an AUC of 0.770 while that of 100-runs is 0.767. The results of the 100-runs are in Table \ref{tab:100runs-result-std}.

\begin{table*}[h!!]
\caption{100-runs results with standard deviation. The best performing attack(s) on each explanation method is(are) highlighted in bold, and the second best attack(s) is(are) underlined.}
\label{tab:100runs-result-std}
\centering
\begin{adjustbox}{width=1\textwidth}
\small
\begin{tabular}{clp{1.4cm}<{\centering}p{1.4cm}<{\centering}p{1.4cm}<{\centering}p{1.4cm}<{\centering}p{1.4cm}<{\centering}p{1.4cm}<{\centering}    p{1.4cm}<{\centering}p{1.4cm}<{\centering}p{1.4cm}<{\centering}p{1.4cm}<{\centering}p{1.4cm}<{\centering}p{1.4cm}<{\centering}}
\toprule
                \multicolumn{1}{l}{$Exp$} & 
                \multirow{1}{*}{\textbf{Attack}}& 
                \multicolumn{2}{c}{\textbf{\cora}} & \multicolumn{2}{c}{\textbf{\coraml}} & \multicolumn{2}{c}{\textbf{\bitcoin}} &
                                \multicolumn{2}{c}{\textbf{\cseer}} & \multicolumn{2}{c}{\textbf{\pubmed}} & \multicolumn{2}{c}{\textbf{\credit}}\\ 
                \cmidrule(r){3-4} \cmidrule(r){5-6} \cmidrule(r){7-8}   
                \cmidrule(r){9-10} \cmidrule(r){11-12} \cmidrule(r){13-14} 
                         &  & AUC           & AP            & AUC      & AP & AUC      & AP 
                          & AUC           & AP            & AUC      & AP & AUC      & AP\\ \midrule

\parbox[t]{2mm}{\multirow{4}{*}{\rotatebox[origin=c]{90}{\grad}}}
& \faac & $0.740 \pm 0.05$	& $0.780 \pm 0.04$ & $0.639 \pm 0.07$	& $0.703 \pm 0.05$ & \underline{$0.474 \pm 0.10$}	& \underline{$ 0.503 \pm 0.07$}
& $0.814 \pm 0.04$	& $0.862 \pm 0.03$ & $0.666 \pm 0.03$	& $0.728 \pm 0.02$ & $0.825 \pm 0.03$	& $0.868 \pm 0.02$\\
& \faaw  & $ 0.682 \pm 0.05$	& $ 0.742 \pm 0.04$ & $ 0.667 \pm 0.07$	& $0.727 \pm 0.06$ & $0.283 \pm 0.03$	& $0.392 \pm 0.02$
& $0.750 \pm 0.05$	& $0.806 \pm 0.04$ & $0.576 \pm 0.02$	& $0.627 \pm 0.02$ & $0.820 \pm 0.02$	& $0.861 \pm 0.01$\\
& \exhoney  & $\mathbf{0.949 \pm 0.02}$	& $\mathbf{0.955 \pm 0.01}$ & $\mathbf{0.801 \pm 0.07}$	& $\mathbf{0.836 \pm 0.06}$ & $\mathbf{0.563 \pm 0.15}$	& $\mathbf{0.552 \pm 0.14}$
& $\mathbf{0.969 \pm 0.02}$	& $\mathbf{0.972 \pm 0.02}$ & $\mathbf{0.787 \pm 0.02}$	& $\mathbf{0.855 \pm 0.02}$ & $\mathbf{0.878 \pm 0.02}$	& $\mathbf{0.916 \pm 0.01}$\\ 
& \exohoneyex & \underline{$0.919 \pm 0.04$}	& \underline{$0.935 \pm 0.02$} & \underline{$0.697 \pm 0.07$}	& \underline{$0.767 \pm 0.05$} & $0.265 \pm 0.04$	& $0.380 \pm 0.02$
& \underline{$0.946 \pm 0.03$}	& \underline{$0.955 \pm 0.02$} & \underline{$0.767 \pm 0.03$}	& \underline{$0.845 \pm 0.02$} & \underline{$0.868 \pm 0.02$}	& \underline{$0.902 \pm 0.01$}\\ \midrule

\parbox[t]{2mm}{\multirow{4}{*}{\rotatebox[origin=c]{90}{\gradinput}}}
& \faac & $0.727 \pm 0.05$	& $0.767 \pm 0.04$ & $0.666 \pm 0.06$	& $0.719 \pm 0.05$ & \underline{$0.523 \pm 0.08$}	& \underline{$ 0.528 \pm 0.05$}
& $0.813 \pm 0.04$	& $0.860 \pm 0.04$ & $0.578 \pm 0.02$	& $0.628 \pm 0.02$ & $0.824 \pm 0.03$	& $0.867 \pm 0.02$\\
& \faaw  & $ 0.689 \pm 0.02$	& $ 0.741 \pm 0.02$ & $ 0.699 \pm 0.06$	& $0.747 \pm 0.06$ & $0.289 \pm 0.06$	& $0.403 \pm 0.04$
& $0.738 \pm 0.04$	& $0.797 \pm 0.04$ & $0.573 \pm 0.02$	& $0.625 \pm 0.02$ & $0.810 \pm 0.02$	& $0.860 \pm 0.01$\\
& \exhoney  & $\mathbf{0.949 \pm 0.02}$	& $\mathbf{0.953 \pm 0.02}$ & $\mathbf{0.792 \pm 0.07}$	& $\mathbf{0.830 \pm 0.06}$ & $\mathbf{0.557 \pm 0.05}$	& $\mathbf{0.547 \pm 0.03}$
& $\mathbf{0.969 \pm 0.02}$	& $\mathbf{0.973 \pm 0.02}$ & $\mathbf{0.784 \pm 0.02}$	& $\mathbf{0.854 \pm 0.02}$ & $\mathbf{0.879 \pm 0.02}$	& $\mathbf{0.918 \pm 0.01}$\\ 
& \exohoneyex & \underline{$0.900 \pm 0.04$}	& \underline{$0.921 \pm 0.03$} & \underline{$0.706 \pm 0.06$}	& \underline{$0.772 \pm 0.05$} & $0.242 \pm 0.05$	& $0.370 \pm 0.02$
& \underline{$0.945 \pm 0.03$}	& \underline{$0.958 \pm 0.02$} & \underline{$0.734 \pm 0.03$}	& \underline{$0.834 \pm 0.02$} & \underline{$0.858 \pm 0.02$}	& \underline{$0.892 \pm 0.02$}\\ \midrule

\parbox[t]{2mm}{\multirow{4}{*}{\rotatebox[origin=c]{90}{\szorro}}}
& \faac & $\mathbf{0.909 \pm 0.03}$	& $\mathbf{0.924 \pm 0.02}$ & \underline{$0.750 \pm 0.06$}	& \underline{$0.790 \pm 0.05$} & \underline{$0.573 \pm 0.04$}	& $\mathbf{0.569 \pm 0.03}$
& \underline{$0.935 \pm 0.02$}	& $0.947 \pm 0.02$ & $0.663 \pm 0.02$	& \underline{$0.730 \pm 0.02$} & \underline{$0.826 \pm 0.03$}	& \underline{$0.868 \pm 0.02$}\\
& \faaw  & $ 0.796 \pm 0.05$	& $ 0.821 \pm 0.05$ & $ 0.717 \pm 0.06$	& $0.744 \pm 0.06$ & $0.512 \pm 0.09$	& $0.512 \pm 0.06$
& $0.826 \pm 0.06$	& $0.853 \pm 0.05$ & $0.559 \pm 0.03$	& $0.606 \pm 0.02$ & $\mathbf{0.829 \pm 0.03}$	& $\mathbf{0.869 \pm 0.02}$\\
& \exhoney  & \underline{$0.902 \pm 0.02$}	& \underline{$0.922 \pm 0.02$} & $\mathbf{0.772 \pm 0.06}$	& $\mathbf{0.811 \pm 0.05}$ & $\mathbf{0.602 \pm 0.03}$	& \underline{$0.568 \pm 0.03$}
& $\mathbf{0.953 \pm 0.03}$	& $\mathbf{0.961 \pm 0.03}$ & $\mathbf{0.686 \pm 0.02}$	& $\mathbf{0.754 \pm 0.02}$ & $0.818 \pm 0.03$	& $0.867 \pm 0.02$\\ 
& \exohoneyex & $0.889 \pm 0.04$	& $0.915 \pm 0.02$ & $0.745 \pm 0.07$	& $0.787 \pm 0.06$ & $0.547 \pm 0.04$	& $0.546 \pm 0.04$
& $0.933 \pm 0.03$	& \underline{$0.948 \pm 0.03$} & \underline{$0.664 \pm 0.02$}	& $0.730 \pm 0.01$ & $0.785 \pm 0.02$	& $0.841 \pm 0.01$\\ \midrule

\parbox[t]{2mm}{\multirow{4}{*}{\rotatebox[origin=c]{90}{\glime}}}
& \faac & \underline{$0.655 \pm 0.06$}	& \underline{$0.716 \pm 0.05$} & \underline{$0.614 \pm 0.06$}	& \underline{$0.655 \pm 0.05$} & $\mathbf{0.537 \pm 0.07}$	& $\mathbf{0.520 \pm 0.05}$
& \underline{$0.762 \pm 0.05$}	& \underline{$0.802 \pm 0.05$} & \underline{$0.614 \pm 0.04$}	& \underline{$0.643 \pm 0.03$} & \underline{$0.827 \pm 0.03$}	& \underline{$0.869 \pm 0.02$}\\
& \faaw  & $ 0.518 \pm 0.05$	& $ 0.540 \pm 0.05$ & $ 0.559 \pm 0.06$	& $0.564 \pm 0.05$ & $0.232 \pm 0.03$	& $0.358 \pm 0.01$
& $0.566 \pm 0.04$	& $0.635 \pm 0.05$ & $0.529 \pm 0.07$	& $0.513 \pm 0.06$ & $0.791 \pm 0.01$	& $0.855 \pm 0.01$\\
& \exhoney  & $\mathbf{0.743 \pm 0.04}$	& $\mathbf{0.784 \pm 0.03}$ & $\mathbf{0.649 \pm 0.06}$	& $\mathbf{0.697 \pm 0.05}$ & \underline{$0.516 \pm 0.05$}	& \underline{$0.512 \pm 0.03$}
& $\mathbf{0.853 \pm 0.05}$	& $\mathbf{0.887 \pm 0.05}$ & $\mathbf{0.680 \pm 0.03}$	& $\mathbf{0.698 \pm 0.03}$ & $\mathbf{0.844 \pm 0.03}$	& $\mathbf{0.889 \pm 0.02}$\\ 
& \exohoneyex & $0.534 \pm 0.05$	& $0.565 \pm 0.05$ & $0.523 \pm 0.05$	& $0.542 \pm 0.05$ & $0.274 \pm 0.05$	& $0.373 \pm 0.02$
& $0.581 \pm 0.05$	& $0.620 \pm 0.04$ & $0.578 \pm 0.05$	& $0.569 \pm 0.05$ & $0.790 \pm 0.03$	& $0.830 \pm 0.04$\\ \midrule

\parbox[t]{2mm}{\multirow{4}{*}{\rotatebox[origin=c]{90}{\gnnexp}}}
& \faac & $0.628 \pm 0.04$	& $0.648 \pm 0.05$ & $0.642 \pm 0.06$	& \underline{$0.695 \pm 0.05$} & \underline{$0.476 \pm 0.09$}	& \underline{$ 0.497 \pm 0.05$}
& $0.628 \pm 0.05$	& $0.704 \pm 0.04$ & $0.555 \pm 0.02$	& $0.596 \pm 0.02$ & $\mathbf{0.825 \pm 0.03}$	& \underline{$0.868 \pm 0.02$}\\
& \faaw  & \underline{$0.730 \pm 0.05$}	& \underline{$0.768 \pm 0.05$} & $\mathbf{0.645 \pm 0.06}$	& $\mathbf{0.699 \pm 0.05}$ & $0.393 \pm 0.11$	& $0.453 \pm 0.05$
& \underline{$0.805 \pm 0.05$}	& \underline{$0.854 \pm 0.04$} & \underline{$0.564 \pm 0.02$}	& \underline{$0.613 \pm 0.02$} & $0.819 \pm 0.02$	& $0.863 \pm 0.02$\\
& \exhoney  & $\mathbf{0.762 \pm 0.04}$	& $\mathbf{0.794 \pm 0.03}$ & \underline{$0.644 \pm 0.07$}	& $0.684 \pm 0.06$ & $\mathbf{0.601 \pm 0.04}$	& $\mathbf{0.574 \pm 0.03}$
& $\mathbf{0.856 \pm 0.05}$	& $\mathbf{0.887 \pm 0.04}$ & $\mathbf{0.616 \pm 0.02}$	& $\mathbf{0.644 \pm 0.02}$ & \underline{$0.818 \pm 0.02$}	& $\mathbf{0.868 \pm 0.02}$\\ 
& \exohoneyex & $0.515 \pm 0.05$	& $0.542 \pm 0.04$ & $0.500 \pm 0.06$	& $0.518 \pm 0.05$ & $0.311 \pm 0.08$	& $0.407 \pm 0.04$
& $0.593 \pm 0.05$	& $0.647 \pm 0.05$ & $0.499 \pm 0.02$	& $0.507 \pm 0.02$ & $0.766 \pm 0.03$	& $0.816 \pm 0.03$\\ \bottomrule

\end{tabular}
\end{adjustbox}
\end{table*}

\section{Hyperparameter settings}
\label{sec:hyperparam}
In Table \ref{tab:hyperparam}, we provide the hyperparameters used in training the different modules of our attack for reproducibility. We used hyperparameters from \cite{fatemi2021slaps}, which are already fine-tuned for these datasets.

\begin{table*}
\centering
\caption{Hyperparameters of the different modules of the attack and target model}
\label{tab:hyperparam}
\begin{tabular}{llll} 
\toprule
\textbf{Param/Module} & \textbf{Self-supervision} & \textbf{Classification} & \textbf{Target Model} \\ \midrule
num\_layers        & 2    & 2     & 2    \\
learning rate      & 0.01 & 0.001 & 0.01 \\
epochs             & 2000 & 200   & 200  \\
hidden\_size       & 512  & 32    & 32   \\
dropout            & 0.5  & 0.5   & 0.5  \\
noisy\_mask\_ratio & 20   & --    & --   \\
weight\_decay      & --   & --    & 5e-4 \\ \bottomrule
\end{tabular}
\end{table*}

\section{Limitations and Future Direction}
This section highlights some limitations and future directions for our work.

\mpara{Partial Explanations.} In the current work, the 30\% nodes we sampled to generate the fixed subgraph for the partial explanation experiment were only sampled once. Such a sampling approach might be subject to sampling bias. Hence, sampling the subgraph multiple times may be desirable to avoid such bias.

\mpara{More Work on Attacks and Defenses.} A promising direction will be to launch more attacks and defenses on other explanation outputs, such as nodes or subgraph-level explanations. These are natural derivatives of graph-based explanation methods, and it is worth investigating to quantify the risk involved in releasing explanations. More defenses that optimize the usefulness of the explanations (fidelity and sparsity) are desirable.

\mpara{Fidelity-sparsity Analysis.} In this work, we analyzed the effect of the usefulness of explanations as measured by their fidelity and sparsity. More work can be done in this direction. In general, our fidelity and sparsity metrics can be complemented with other domain and dataset-specific metrics of explanation utility.

\end{document}